\pdfoutput=1
\documentclass[pdflatex, 10pt]{article}

\usepackage{amsfonts,amsmath,amsthm,amssymb}
\usepackage{bm,color}
\usepackage{graphicx}
\usepackage[margin=1in]{geometry}
\usepackage{mymcr}

\usepackage{subfigure}
\usepackage{braket}
\usepackage{hyperref}

\theoremstyle{plain}
\newtheorem{theo}{Theorem}
\newtheorem{lemm}[theo]{Lemma}

\title{Simultaneous Safe Screening of Features and Samples \\ in Doubly Sparse Modeling}

\date{\today}

\author{
Atsushi Shibagaki \\
Nagoya Institute of Technology \\
\texttt{shibagaki.a.mllab.nit@gmail.com} \\
\and
Masayuki Karasuyama \\
Nagoya Institute of Technology \\
\texttt{karasuyama@nitech.ac.jp} \\
\and
Kohei Hatano\\
Kyushu University \\
\texttt{hatano@inf.kyushu-u.ac.jp} \\
\and
Ichiro Takeuchi\thanks{Corresponding author} \\
Nagoya Institute of Technology \\
\texttt{takeuchi.ichiro@nitech.ac.jp} \\
}

\begin{document}

\maketitle

\begin{abstract}
  The problem of learning a sparse model
 is conceptually interpreted
 as the process of identifying \emph{active} features/samples
 and
 then
 optimizing the model
 over them.
 %
 %It is,
 %however,
 %difficult to perfectly identify active features/samples
 %before
 %we obtain the optimal solution.
 %
 Recently introduced
 \emph{safe screening}
 allows us to identify a part of non-active features/samples.
 So far,
 safe screening has been individually studied
 either
 for feature screening
 or
 for sample screening.
 In this paper,
 we introduce a new approach for
 safely screening features and samples \emph{simultaneously}
 by
 alternatively
 iterating feature and sample screening steps.
 A significant advantage of
 considering them
 simultaneously
 rather than individually
 is that
 they have a \emph{synergy} effect
 in the sense that
 the results of the previous safe feature screening can be exploited
 for improving the next safe sample screening performances,
 and vice-versa.
 We first theoretically investigate the synergy effect,
 and then illustrate the practical advantage through intensive numerical experiments
 for problems
 with large numbers of features and samples.
 %
 %In addition,
 %simultaneously considering safe feature and sample screenings
 %allows us to introduce method for predicting active features/samples
 %without \emph{false positive errors}.

 \vspace{.1in}
 \begin{center}
 {\bf Keywords}
 \end{center}
 Sparse learning,
Safe screening,
Safe keeping,
Support vector machines, 
Convex optimization, 
Duality gap

 \vspace{.1in}
\end{abstract}

\clearpage

\section{Introduction}
\label{sec:introduction}
In many areas of science and industry,
large-scale datasets
with many features and samples
are collected and analyzed
for data-driven scientific discovery and decision making.
One promising approach to
handle large numbers of features and samples
analyze these large-scale datasets
is introducing \emph{sparsity} constraints in statistical models
\cite{hastie2015statistical}.
The most common approach for inducing \emph{feature sparsity},
e.g.,
in a linear least-square model fitting,
is using sparsity-inducing penalties such as $L_1$-norm of the coefficients
\cite{tibshirani1996regression}.
A feature sparse model
depends only on a subset of features
(called \emph{active} features),
and the rest of the features
(called \emph{non-active} features)
are irrelevant.
On the other hand,
the most popular machine learning algorithm
that induces
\emph{sample sparsity} would be
the support vector machine (SVM)
\cite{boser1992training}.
In the SVM,
the large margin principle
%(equivalently, the use of hinge loss function)
enhances sample sparsity
in the sense that
it depends only on a subset of samples
(called \emph{support vectors (SVs)} or \emph{active} samples)
and does not depend on the rest of the samples
(called \emph{non-SVs} or \emph{non-active} samples).

The problem of training a sparse model
can be conceptually divided into two tasks.
The first task is to identify
active features or samples,
and
the second task is to optimize the model
with respect to them.
If the first task is perfectly completed,
then the second task would be relatively easy
because
one only needs to solve a small-scale optimization problem
that only depends on a subset of features or samples.
In general,
however,
it is impossible to completely identify the set of active features or samples
before actually training the model.
Thus,
some of existing sparse learning solvers employ
\emph{working set} approaches.
Roughly speaking,
a working set is the set of features or samples
that are predicted to be active.
Since the prediction is not perfect,
there may be false positives
(non-active features/samples in the working set)
and
false negatives
(active features/samples not in the working set).
Thus,
one needs to repeat
a working set prediction
and
optimization over the working set
until the optimality condition is satisfied.
For example,
LIBSVM~\cite{chang2011libsvm},
a well-known SVM solver,
employs working set approaches.
It repeats
predicting a working set (the set of samples that are predicted to be SVs),
and
optimizing the model by only using the samples in the working set.

The problem of learning a sparse model
is conceptually interpreted
as the process of identifying active features or samples
and
then
optimizing the model
over them.
Numerical optimization of
those sparse learning methods is often
computationally expensive.
A major difficulty stems from a
combinatorial property of
identifying active features or samples for which
\emph{working set} approaches have been studied to predict a set of
features or samples to be active.
Recently,
a new approach
called
\emph{safe screening}
has been studied by several authors.
%\cite{ghaoui2012safe,xiang2011learning,wang2013lasso,bonnefoy2014dynamic,liu2014safe,wang2014safe,xiang2014screening,fercoq2015mind,ndiaye2015gap,ogawa2013safe,ogawa2014safe,wang2013scaling,Zimmert2015}.
%
Safe screening
enables to identify a subset of non-active features or samples
before or during the model training process.
A nice thing about
safe screening
is that
it is guaranteed to have no false negatives,
i.e.,
safe screening never identify active features or samples as non-active.
It means that,
if we train the model
by only using the remaining set of features or samples after safe screening,
the solution is guaranteed to be optimal.
The basic technical idea behind safe screening is
to bound the solution of the problem
within a
% compact
 region,
and show that some features or samples cannot be active
wherever the optimal solution is located within the
% compact
 region.

After the seminal work by
% El Ghaoui et al.
\cite{ghaoui2012safe},
\emph{safe feature screening}
(safely screening a part of non-active features in sparse feature models such as LASSO)
has been intensively studied in the literature
\cite{xiang2011learning,wang2013lasso,bonnefoy2014dynamic,liu2014safe,wang2014safe,xiang2014screening,fercoq2015mind,ndiaye2015gap}.
Safe feature screening
exploits
the fact that
the sparseness of a feature is characterized by a property of the dual solution,
i.e.,
if the dual solution satisfies a certain condition in the dual space,
the feature is guaranteed to be non-active.
%
%Several authors proposed different ways of constructing a
%% compact
% region
%in the dual space.
%
Safe feature screening is
beneficial
especially when the number of features is large.
There are also several studies on
\emph{safe sample screening}
(safely screening a part of non-active samples in sparse sample models such as the SVM)
\cite{ogawa2013safe,wang2013scaling,Zimmert2015}.
The basic idea behind safe sample screening is that
the sparseness of a sample is characterized by a property of the primal solution.
If
the primal solution
satisfies
a certain condition in the primal space,
the sample is guaranteed to be non-active.
Safe sample screening is
% particularly
useful
when the number of samples is large.

In this paper,
we consider problems
where
the numbers of features and samples are both large.
In these problems,
we consider
a class of learning algorithms
that induce
both of feature sparsity and sample sparsity,
which we call
\emph{doubly sparse modeling}.
Our main contribution
in this paper is
to develop a safe screening method
that can identify both of non-active features and non-active samples
simultaneously
in doubly sparse modeling.
Specifically,
we propose a novel method for simultaneously constructing two
% compact
 regions,
one in the dual space and the other in the primal space.
The former is used for safe feature screening,
while
the latter is used for safe sample screening.

A significant advantage of considering safe feature screening and safe sample screening
simultaneously
rather than individually
is that
they have a \emph{synergy} effect.
Specifically,
we show that,
after we know that a part of features are non-active
based on safe feature screening,
we can potentially improve the performance of safe sample screening.
Our basic idea behind this property is that,
by fixing a part of the primal variables,
the
% compact
region in the primal space,
which is used for safe sample screening,
can be made smaller
(we can also show the converse similarly).
%
%Based on a similar idea,
%we can also show the converse
%i.e.,
%safe sample screening can potentially improve the performances of the safe feature screening.
%
\begin{figure*}[t]
  \begin{center}
   \includegraphics[clip,width=12.5cm]{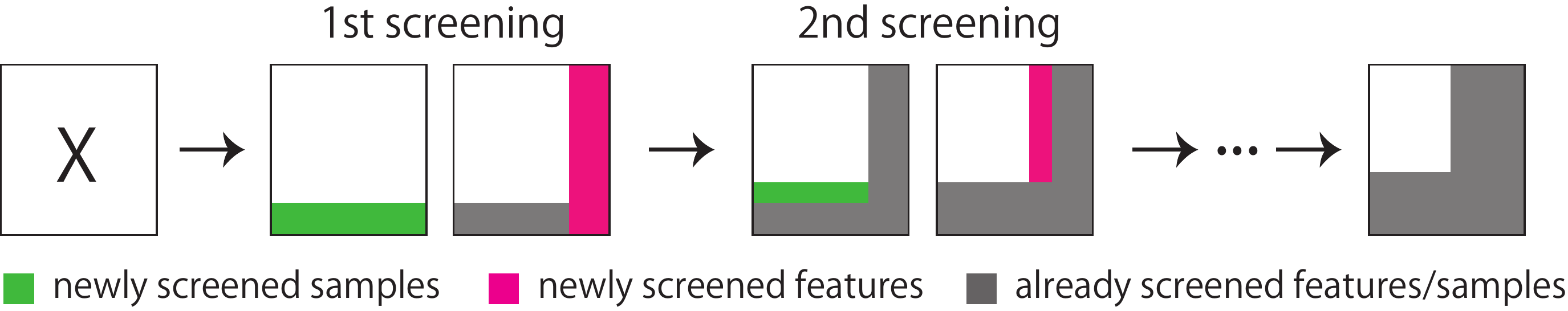}
   \caption{
   Schematic illustration of the proposed approach.
   By iterating safe feature screening and safe sample screening,
   irrelevant features and samples are safely removed out from the training set.
   A significant advantage of considering them simultaneously
   rather than individually is that
   they have a \emph{synergy} effect
   in the sense that
   the screening performances
   (the number of features/samples that can be safely removed out)
   are gradually improved
   by exploiting the results in earlier steps.
   }
   \label{fig:ills}
  \end{center}
\end{figure*}

These findings suggest that,
safe screening performances of features and samples can be both improved
by alternatively repeating them.
%safe feature screening
%and
%safe sample screening.
\figurename~\ref{fig:ills} is a schematic illustration of the simultaneous safe screening approach.

Another interesting finding we first introduce in this paper is that,
by simultaneously considering
% compact
 regions
both in the dual and the primal spaces,
we can also identify features and samples that are guaranteed to be active.
We call this technique as \emph{safe keeping}.
While
safe screening assures no false negative findings,
safe keeping guarantees no false positive findings of active features/samples.
By combining these two techniques,
we can better identify active features/samples.
A practical advantage of safe keeping is that
we do not have to consider the safe screening rules anymore
for features and samples
which are identified as active
by safe keeping.
%because
%we already know that
%these features and samples are never screened out.
%
This is helpful for reducing the computational costs of safe screening rule evaluations
especially
in the context of
dynamic safe screening \cite{bonnefoy2014dynamic}.
% where safe screening rules are evaluated many times
% during the optimization process.
% , e.g., every 10 iterations.

\subsection{Notation and outline}
\label{subsec:notation-outline}
% \paragraph{Notation:}
We use the following notations in the rest of the paper.
For any natural number $n$,
we define
$[n] := \{1, \ldots, n\}$.
For an
$n \times d$ matrix $M$,
its
$i$-th row
and
$j$-th column
are denoted as
$M_{i:}$
and
$M_{:j}$,
respectively,
for
$i \in [n]$
and
$j \in [d]$.
The
$L_1$ norm,
the
$L_2$ norm
and
the
$L_\infty$ norm
of a vector
$v \in \RR^m$
are
respectively
written as
$\|v\|_1 := \sum_{k \in [m]} |v_k|$,
$\|v\|_2 := \sqrt{\sum_{k \in [m]} |v_k|^2}$
and
$\|v\|_\infty := \max_{k \in [m]} |v_k|$.
For a scalar $z$,
we define
$[z]_+ := \max\{0, z\}$.
We write the subdifferential operator as
$\partial$,
and
remind that
the subdifferential of $L_1$ norm is given as
$\partial \|v\|_1 = \{g \mid \|g\|_\infty \le 1, g^\top v = \|v\|_1 \}$.
For a function $f$,
we denote its domain as
${\rm dom}f$.

Here is the outline of the paper.
\S\ref{sec:preliminaries}
introduces
the problem formulation we mainly consider in this paper
and
basic concepts of safe screening.
\S\ref{sec:simultaneous}
describes
our main contribution
where
we show that
simultaneously screening features and samples has synergetic advantages.
\S\ref{sec:safe-keeping}
also presents another contribution
where
we introduce a new approach
for predicting active set without false positive errors.
\S\ref{sec:LP-based}
discusses another problem
that induces spasities both in features and samples.
\S\ref{sec:experiments}
covers
numerical experiments
for demonstrating the advantage of the simultaneous safe screening.
\S\ref{sec:conclusions}
concludes the paper.

\section{Preliminaries}
\label{sec:preliminaries}
In this section,
we first describe the problem formulation
in \S\ref{subsec:problem-formulation}.
Then,
we briefly summarize the basic concepts of
safe feature screening
and
safe sample screening
in
\S\ref{subsec:safe-feature-screening}
and
\S\ref{subsec:safe-sample-screening},
respectively.
\subsection{Problem formulation}
\label{subsec:problem-formulation}
Consider classification and regression problems
with
the number of samples $n$
and
the number of features $d$.
The training set is written as
$\{(x_i, y_i)\}_{i \in [n]}$
where
$x_i \in \RR^d$,
and
$y_i \in \{-1, 1\}$
for classification
and
$y_i \in \RR$
for regression.
The $n \times d$ input data matrix (design matrix)
is denoted as
$X := [x_1, \ldots, x_n]^\top$.
% , and its $j$-th column is written as
% $X_{: j} \in \RR^n$
% for
% $j \in [d]$.

We consider a linear classification and regression function
in the form of
$f(x) = x^\top w$,
and study the problem of estimating the parameter
$w \in \RR^d$
by solving a class of regularized empirical risk minimization problems:
\begin{align}
 \label{eq:primal_prob}
 \min_{w \in \RR^d}P_\lambda(w)
% \!\!
:=
%\!\!
% \min_{w \in \RR^d}
% \left\{
 \lambda \psi(w)
 +
 \frac{1}{n}
 \sum_{i \in [n]}
 \ell_i(x_i^\top w),
% \right\},
\end{align}
% \vspace*{-0.5em}
where
$\psi$
is a penalty function,
% $\ell_i(x_i^\top w) := \ell(y_i,x_i^\top w)$
$\ell_i$
is a loss function
% $\ell$
for the $i$-th sample\footnote{
Here,
we use individual loss function
$\ell_i$
for
$i \in [n]$
because it implicitly depends on
$y_i$
 (see, e.g., \eq{eq:smoothed_hinge} and \eq{eq:soft_insensitive}).
},
and
$\lambda > 0$
is a trade-off parameter
for controlling the balance between the penalty and the loss.

%As we discussed in \S\ref{sec:introduction},
In this paper,
we study doubly sparse modeling,
i.e.,
a pair of penalty and loss function
that induces sparsities both in features and samples.
As specific working examples,
we consider
$L_1$-penalized smoothed hinge SV classification
and
$L_1$-penalized smoothed $\veps$-insensitive SV regression.
The use of these smoothed loss functions
% make numerical optimizations easy to handle and stable, while
% producing
are known to produce almost same solutions
as the original hinge or $\veps$-insensitive loss functions
\cite{shalev2015accelerated}.
We will discuss other doubly sparse modeling problem
in
\S\ref{sec:LP-based}.

Remembering that the original SVMs
are trained with $L_2$ penalty,
by combining an additional $L_1$-penalty,
our penalty function
$\psi$
is written as \emph{elastic net penalty}:
% ~\footnote{Penalty terms in the form of
% \eq{eq:regularization-term} are often called \emph{elastic net penalty}.}
% ~\footnote{
% Although it is possible to introduce a tuning parameter
% for making a balance
% between the $L_1$ penalty and the $L_2$ penalty
% in \eq{eq:regularization-term},
% we do not explicitly consider such a generalization
% just because of notational simplicity.
% %
% All the discussions described below can be straightforwardly generalized
% even if we consider it.
% }
\begin{align}
 \label{eq:regularization-term}
 \psi(w) := \|w\|_1 + \frac{\beta}{2} \|w\|_2^2,
\end{align}
where $\beta > 0$ is a balancing parameter which we omit hereafter
by substituting $\beta = 1$ for notational simplicity.
%
% \vspace*{-1em}
%
The smoothed hinge loss
and
the smoothed $\veps$-insensitive loss
are respectively written as
% ~\footnote{
% The smoothed hinge loss
% \eq{eq:smoothed_hinge}
% and
% the smoothed $\veps$-insensitive loss
% \eq{eq:soft_insensitive}
% are often used
% in the literature
% as alternatives to
% the standard hinge loss
% and
% the standard $\veps$-insensitive loss,
% respectively,
% because
% they are easy to handle in numerical optimization,
% and
% they have been shown to produce almost same solution \cite{shalev2015accelerated}.
% }.
% \vspace*{-1em}
\begin{align}
\label{eq:smoothed_hinge}
\!\!\ell_i(x_i^\top w) \!\! &:= \!\! \begin{cases}
0 & ( y_i x_i^\top w > 1 ), \\
1 - y_i x_i^\top w - \frac{\gamma}{2}  & (y_i x_i^\top w < 1-\gamma), \\
\frac{1}{2\gamma} (1 - y_i x_i^\top w)^2 & ({\rm otherwise}),
\end{cases}\\
% \end{align}
% % \vspace*{-1em}
% and
% \begin{align}
  \label{eq:soft_insensitive}
  \!\!\ell_{i}(x_i^\top w) \!\!&:=\!\! \begin{cases}
    \!0 &\!\!\!\! ( | x_i^\top w - y_i | \!\! < \!\! \veps ), \\
    \!| x_i^\top \! w \! - \! y_i |\!\! - \! \veps \! - \! \frac{\gamma}{2}
    &\!\!\!\!\! (| x_i^\top w \! - \! y_i | \!\! > \!\! \veps \!\! + \!\! \gamma), \\
    \!\frac{1}{2\gamma} ( | x_i^\top w - y_i |\!\! - \veps)^2 &\!\!\!\!\! ({\rm otherwise}),
  \end{cases}
\end{align}
% \vspace*{-1em}
where
$\gamma > 0$
is a tuning parameter.
There are profiles the smoothed hinge loss and
the smoothed $\veps$-insensitive loss as shown
in \figurename~\ref{fig:loss_func}.
\begin{figure}[t]
  \begin{center}
   \includegraphics[clip,width=10cm]{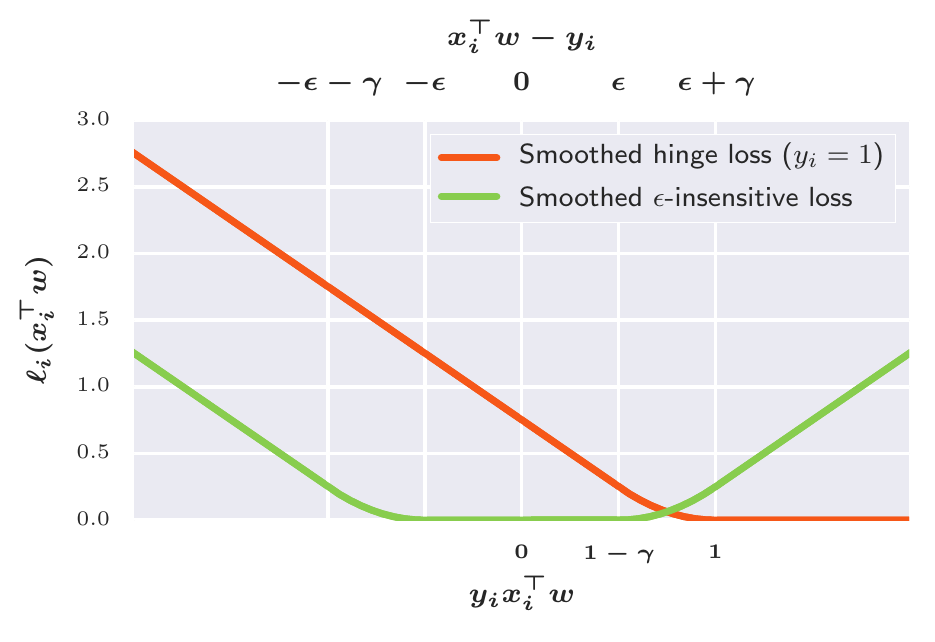}
   \caption{Illustrations of the smoothed hinge loss (orange)
           and
          the smoothed $\veps$-insensitive loss (green), where
          $\gamma = 0.5$ and $\veps = 0.5$.
   }
   \label{fig:loss_func}
  \end{center}
\end{figure}

Using Fenchel's duality theorem (see, e.g., Corollary 31.2.1 in \cite{rockafellar1970convex}),
the dual problem of
\eq{eq:primal_prob}
is written as
\begin{align}
  \label{eq:dual_prob}
 \!\!\!\!\!
 \max_{\alpha} D_\lambda(\alpha)  :=
 -\lambda \psi^* \!\!\left( \frac{1}{\lambda n}
   X^\top \alpha  \right) \!\!
- \frac{1}{n} \! \sum_{i \in [n]} \! \ell_{i}^* (-\alpha_i),
%  \!\!\!\!\max_{\alpha} D_\lambda(\alpha) \!\! := \!\!
 % \max_{\alpha \in \RR^n} \!\!
%  \left\{ \!\!
% -\lambda \psi^* \!\! \left(\!\!\frac{1}{\lambda n}
   % \sum_{i=1}^{n} \alpha_i x_i
%   X^\top \!\! \alpha \!\!
% \right) \!\!
%  -\!\! \frac{1}{n}\!\! \sum_{i \in [n]}\!\! \ell_{i}^* (-\alpha_i) \!\! \right\},
\end{align}
where
$\psi^*$
and
$\ell^*$
is convex conjugate function of
$\psi$
and
$\ell$,
respectively.
The convex conjugate function of the penalty term in
\eq{eq:regularization-term}
is given as
% $
%   \psi^*(v) = \frac{1}{2} \sum_{j=1}^d
%    ( [| v_j | - 1]_+  )^2.
% $
(see \cite{shalev2015accelerated} )
\begin{align*}
  \psi^*(v) &= \frac{1}{2} \sum_{j=1}^d
   ( [| v_j | - 1]_+  )^2.
\end{align*}
The convex conjugate functions of
the smoothed hinge loss
and
the smoothed $\veps$-insensitive loss
are
respectively written as
% $
%   \ell^*_{i}(\alpha_i) =
%     \frac{\gamma}{2} \alpha_i^2 + y_i \alpha_i
% $ for
% $ y_i \alpha_i \in [-1, 0] $
% and
% $ \infty $ otherwise,
\begin{align}
  \label{eq:conj_smoothed_hinge}
  \ell^*_{i}(\alpha_i) &= \begin{cases}
    \frac{\gamma}{2} \alpha_i^2 + y_i \alpha_i & (y_i \alpha_i \in [-1, 0]), \\
    \infty & ({\rm otherwise}),
  \end{cases}
\end{align}
and
% $
%   \ell^*_{i}(\alpha_i) =
%     \frac{\gamma}{2} \alpha_i^2 + y_i \alpha_i + \veps | \alpha_i |
% $ for
% $\alpha_i \in [-1, 1] $ and
% $ \infty $ otherwise.
\begin{align}
  \label{eq:conj_soft_insensitive}
  \ell^*_{i}(\alpha_i) &= \begin{cases}
    \frac{\gamma}{2} \alpha_i^2 + y_i \alpha_i + \veps | \alpha_i |  & (\alpha_i \in [-1, 1]), \\
    \infty & ({\rm otherwise}).
  \end{cases}
\end{align}
We call the problems in
\eq{eq:primal_prob}
and
\eq{eq:dual_prob}
as
\emph{primal problem}
and
\emph{dual problem},
respectively,
and denote the primal optimal solution as
$w^* \in \RR^d$
and
the dual optimal solution
as
$\alpha^* \in \RR^n$.
%
% In dual form,
% the linear classification and regression function is written as
% {\bf [CHECK]}
% $f(x) = xxx$.
% {\bf [CHECK]}
\subsection{Safe feature screening}
\label{subsec:safe-feature-screening}
The goal of safe feature screening is to identify
a part of non-active features
$\{j \in [d] \mid w_j^* = 0\}$
before or during the optimization process.
Safe feature screening
is built on the following KKT optimality condition (see Theorem 31.3 in \cite{rockafellar1970convex})
\begin{align}
  \label{eq:kkt1}
  \frac{1}{\lambda n} X^\top \alpha^*  &\in \partial \psi(w^*).
\end{align}
In the case of our specific regularization term
\eq{eq:regularization-term},
the optimality condition
\eq{eq:kkt1}
is written as
\begin{align}
  \label{eq:kkt1_elastic_net}
 \frac{1}{\lambda n} X_{:j}^\top \alpha^* \in \begin{cases}
    \cfrac{w^*_j}{|w^*_j|} + w^*_j & (w^*_j \neq 0), \\
    [-1,1] & (w^*_j = 0).
  \end{cases}
\end{align}
The optimality condition
\eq{eq:kkt1_elastic_net}
indicates that
\begin{align*}
 |X_{:j}^\top \alpha^*| \le \lambda n
 ~\Rightarrow~
 w^*_j = 0.
\end{align*}
The basic idea behind safe feature screening is to construct a region
$\Theta_{\alpha^*} \subset \RR^n$
in the dual space
so that
$\alpha^* \in \Theta_{\alpha^*}$,
and then compute an upper bound
$UB(|X_{:j}^\top \alpha^*|) := \max_{\alpha \in \Theta_{\alpha^*}} |X_{:j}^\top \alpha|$.
Using this upper bound,
we can construct a safe feature screening rule in the form of
\begin{align*}
  UB(|X_{:j}^\top \alpha^*|) \le \lambda n
 ~\Rightarrow~
 w_j^* = 0.
\end{align*}

After the seminal work \cite{ghaoui2012safe},
many different approaches for constructing a region
$\Theta_{\alpha^*}$
have been developed (see \S\ref{sec:introduction}).
%\cite{xiang2011learning,wang2013lasso,liu2014safe,wang2014safe,xiang2014screening,fercoq2015mind,ndiaye2015gap}.
%
Among those,
we use an approach in \cite{ndiaye2015gap}.
Noting the fact that the convex conjugate function
$\ell^*$
is $\gamma$-strongly convex,
and henceforth the dual objective function
$D_\lambda(\alpha)$
is $\gamma/n$-strongly concave,
we can define a region
%\vspace*{-1em}
\begin{align}
 \label{eq:dual_gap_sphere}
 \Theta_{\alpha^*} := \{
 \alpha \mid \|\hat{\alpha} - \alpha\|_2 \le \sqrt{2 n G_\lambda(\hat{w}, \hat{\alpha})/\gamma} \},
 % ~~~
 % \forall~\hat{w} \in {\rm dom} P_\lambda,
 % ~
 % \forall~\hat{\alpha} \in {\rm dom} D_\lambda,
\end{align}
where
$G_\lambda(\hat{w}, \hat{\alpha}) := P_\lambda(\hat{w}) - D_\lambda(\hat{w})$
is the duality gap defined
by an arbitrary pair of primal feasible solution
$\hat{w} \in {\rm dom}P_\lambda$
and dual feasible solution
$\hat{\alpha} \in {\rm dom}D_\lambda$.
Since the region
$\Theta_{\alpha^*}$
is a sphere,
we can explicitly write the upper bound as
\begin{align}
 \label{eq:UB-for-feature-screening}
 \!\! UB(|X_{:j}^\top \alpha^*|) \!=\!
 |X_{:j}^\top \hat{\alpha}| \! + \!
 \|X_{:j}\|_2 \sqrt{2 n G_\lambda(\hat{w}, \hat{\alpha})/\gamma}.
\end{align}
%
% The lower bound can be also similarly obtained as
% \begin{align}
%  \label{eq:LB-for-feature-screening}
%  UB(|X_{:j}^\top \alpha^*|)
%  =
%  |X_{:j}^\top \hat{\alpha}|
%  -
%  \|X_{:j}\|_2 \sqrt{2 n G_\lambda(\hat{w}, \hat{\alpha})/\gamma}.
% \end{align}
\subsection{Safe sample screening}
\label{subsec:safe-sample-screening}
The goal of safe sample screening is to identify
a part of non-active samples
$\{i \in [n] \mid \alpha_i^* = 0, \pm 1\}$
before or during the optimization process.
Here,
we slightly abuse the word ``non-active''
in the sense that
we call a sample to be non-active
not only when the corresponding
$\alpha^*_i$
is $0$,
but also
when it is
$\pm 1$.
Although
the $i$-th sample can be removed out
only when
$\alpha^*_i = 0$,
we have similar computational advantages
when we can guarantee that
$\alpha^*_i = \pm 1$
because
the size of the optimization problem can be reduced.

Safe sample screening is also built on the KKT optimality condition
\begin{align}
  \label{eq:kkt2}
  x_i^\top w^* &\in \partial \ell_i^*(- \alpha^*_i ).
\end{align}
In the case of smoothed hinge loss,
the KKT condition
\eq{eq:kkt2}
is written
when $y_i = 1$
as
\begin{align}
  \label{eq:kkt2_smoothed_hinge_posi}
  x_i^\top w^* \in \begin{cases}
  [1, \infty) & (\alpha^*_i = 0) \\
  (-\infty, 1 - \gamma] & (\alpha^*_i = 1 ) \\
  - \gamma \alpha_i^* + 1 & (\alpha^*_i \in (0, 1)).
  \end{cases}
\end{align}
%
%In the same way as safe feature screening,
We construct a region
$\Theta_{w^*} \subset \RR^d$
in the primal space
so that
$w^* \in \Theta_{w^*}$,
and then
compute a lower bound
$LB(x_i^\top w^*) := \min_{w \in \Theta_{w^*}} x_i^\top w$
and an upper bound
$UB(x_i^\top w^*) := \max_{w \in \Theta_{w^*}} x_i^\top w$.
The optimality condition
\eq{eq:kkt2_smoothed_hinge_posi}
suggests that
% $
%  LB(x_i^\top w^*) \ge 1
%  ~\Rightarrow~
%  \alpha_i^* = 0,
%  UB(x_i^\top w^*) \le 1 - \gamma
%  ~\Rightarrow~
%  \alpha_i^* = 1.
% $
\begin{align*}
 &
 LB(x_i^\top w^*) \ge 1
 ~\Rightarrow~
 \alpha_i^* = 0,
 \\
 &
 UB(x_i^\top w^*) \le 1 - \gamma
 ~\Rightarrow~
 \alpha_i^* = 1.
\end{align*}
Similarly,
for $y_i = -1$,
the optimality condition is written as
\begin{align}
  \label{eq:kkt2_smoothed_hinge_nega}
  x_i^\top w^* \in \begin{cases}
    (-\infty, -1] & (\alpha^*_i = 0) \\
    [\gamma -1, \infty) & (\alpha^*_i = -1) \\
    -\gamma \alpha^*_i - 1 & (\alpha^*_i \in (-1, 0)),
  \end{cases}
\end{align}
suggesting that
% $
%  UB (y_i x_i^\top w^*) \le -1
%  ~\Rightarrow~
%  \alpha_i^* = 0,
%  LB(y_i x_i^\top w^*) \ge \gamma - 1
%  ~\Rightarrow~
%  \alpha_i^* = -1.
% $
\begin{align*}
 &
 UB (y_i x_i^\top w^*) \le -1
 ~\Rightarrow~
 \alpha_i^* = 0,
 \\
 &
 LB(y_i x_i^\top w^*) \ge \gamma - 1
 ~\Rightarrow~
 \alpha_i^* = -1.
\end{align*}
In the case of
smoothed $\veps$-insensitive loss,
the optimality condition
\eq{eq:kkt2}
is written as
\begin{align}
  \label{eq:kkt2_soft_insensitive}
 x_i^\top w^* \in \begin{cases}
    [y_i - \veps, y_i + \veps ] & (\alpha^*_i = 0), \\
    [\gamma + y_i + \veps, \infty ) & (\alpha^*_i = -1), \\
    (-\infty ,- \gamma + y_i - \veps] & (\alpha^*_i = 1), \\
    -\gamma \alpha^*_i + y_i + \veps & (\alpha^*_i \in (-1, 0)), \\
    -\gamma \alpha^*_i + y_i - \veps & (\alpha^*_i \in (0, 1)).
\end{cases}
\end{align}
It indicates that
% $
%  LB(x_i^\top w^*) \ge y_i - \veps
%  ~\text{ and }~
%  UB(x_i^\top w^*) \le y_i + \veps
%  ~\Rightarrow~
%  \alpha_i^* = 0,
%  LB(x_i^\top w^*) \ge \gamma + y_i + \veps
%  ~\Rightarrow~
%  \alpha_i^* = -1,
%  UB(x_i^\top w^*) \le -\gamma + y_i - \veps
%  ~\Rightarrow~
%  \alpha_i^* = 1.
%  $
\begin{align*}
 &
 LB(x_i^\top w^*) \ge y_i - \veps
 ~\text{ and }~
 UB(x_i^\top w^*) \le y_i + \veps
 ~\Rightarrow~
 \alpha_i^* = 0,
 \\
 &
 LB(x_i^\top w^*) \ge \gamma + y_i + \veps
 ~\Rightarrow~
 \alpha_i^* = -1,
 \\
 &
 UB(x_i^\top w^*) \le -\gamma + y_i - \veps
 ~\Rightarrow~
 \alpha_i^* = 1.
\end{align*}
In order to develop a sphere region
$\Theta_{w^*}$
in the primal space,
we extend the duality GAP-based safe feature screening approach
proposed in
\cite{ndiaye2015gap}
%and
%discussed
%in
%\S\ref{subsec:safe-feature-screening}
into safe sample screening context.
The result is summarized in the following lemma.
\begin{lemm}
 \label{lemm:safe-sample-screening}
 For any
 $\hat{w} \in {\rm dom}P_\lambda$
 and
 $\hat{\alpha} \in {\rm dom}D_\lambda$,
 \begin{align}
  \label{eq:primal_gap_sphere}
  w^* \in \Theta_{w^*} =
  \{ w \mid
  \| \hat{w} - w \|_2 \le
  \sqrt{2 G_{\lambda}(\hat{w}, \hat{\alpha}) / \lambda} \},
  % ~~ \forall \hat{w} \in dom P_{\lambda},
  % ~ \forall \hat{\alpha} \in dom D_{\lambda}.
 \end{align}
 Furthermore,
 using the sphere form region
 $\Theta_{w^*}$
 in
 \eq{eq:primal_gap_sphere},
 for any
 $\hat{w} \in {\rm dom}P_\lambda$
 and
 $\hat{\alpha} \in {\rm dom}D_\lambda$,
 a pair of lower and upper bounds of
 $x_i^\top w^*$
 are given as
\begin{subequations}
\label{eq:bounds-for-sample-screening}
 \begin{align}
  LB(x_i^\top w^*) &= x_i^\top \hat{w} -
  \|x_i\|_2 \sqrt{2 G_{\lambda}(\hat{w}, \hat{\alpha}) / \lambda},  \\
  UB(x_i^\top w^*) &= x_i^\top \hat{w} +
  \|x_i\|_2 \sqrt{2 G_{\lambda}(\hat{w}, \hat{\alpha}) / \lambda }.
 \end{align}
\end{subequations}
\end{lemm}
The proof of
Lemma~\ref{lemm:safe-sample-screening}
is presented in Appendix \ref{app:proofs}.

\section{Simultaneous safe screening}
\label{sec:simultaneous}
We have shown that,
for doubly sparse modeling problems,
%a class of models having both of feature and sample sparsity,
safe screening rules for both of features and samples can be constructed respectively.
This has not been explored in depth regardless of its practical importance.
In this paper
%Based on those rules,
we further develop the framework
in which
safe feature screening
and
safe sample screening
are alternately iterated.
A significant additional benefit of this framework
%
% In this section,
% we describe our main contribution.
% %
% We simply consider an algorithm
% in which
% safe feature screening
% and
% safe sample screening
% are alternately iterated.
% %
% The core idea
is that
the result of the previous safe feature screening
can be exploited
for making the primal region
$\Theta_{w^*}$
smaller,
meaning that
the performance of the next safe sample screening
can be improved,
and vice-versa.
%
%Conversely,
%the result of the previous safe sample screening
%can be exploited
%for making the dual region
%$\Theta_{\alpha^*}$
%smaller
%and hence improving
%the performance of the next safe feature screening.

The following two theorems formally state these ideas.
First,
Theorem~\ref{theo:Sample2Feature}
states that
we can obtain tighter upper bound
for feature screening
by exploiting the result of the previous safe sample screening.
\begin{theo}
 \label{theo:Sample2Feature}
 Consider a safe feature screening problem
 given
 arbitrary pair of primal and dual feasible solution
 $\hat{w} \in {\rm dom}P_\lambda$
 and
 $\hat{\alpha} \in {\rm dom}D_\lambda$.
 Furthermore,
 suppose that
 the result of the previous safe sample screening step assures the optimal values
 $\alpha_i^*$
 for
 a subset of the samples
 $i \in \cS \subset [n]$.
 Let
 $\cU_s := [n] \setminus \cS$,
 $r_D := \sqrt{2 n G_\lambda(\hat{w}, \hat{\alpha})/\gamma}$,
 and
 $\tilde{\alpha}$
 be an $n$-dimensional vector
 whose element is defined as
 $\tilde{\alpha}_i = \hat{\alpha}_i$
 for
 $i \in \cU_s$
 and
 $\tilde{\alpha}_i = \alpha^*_i$
 for
 $i \in \cS$.
Then,
$|X_{:j}^\top \alpha^*|$
is bounded from above by the following upper bound:
% \vspace*{-0.3cm}
 \begin{align}
  \label{eq:new-UB-for-feature-screening}
  \!\!\!\!\!\!
  \tilde{UB}(\!|X_{:j}^\top \alpha^*|\!)
  \!\!:=\!\!
  |X_{:j}^\top \tilde{\alpha}|
  \!+\!
  \|X_{:j \cU_s} \|_2
  \sqrt{
  r_D^2 \!\!- \!\!
  \|\hat{\alpha}_{\cS} \! - \! \alpha^*_{\cS}\|_2^2
  },
 \end{align}
% \vspace*{-0.3cm}
and
the upper bound
% $\tilde{UB}(|X_{:j}^\top \alpha^*|)$
in \eq{eq:new-UB-for-feature-screening}
is tighter than or equal to
that
%$UB(|X_{:j}^\top \alpha^*|)$
in \eq{eq:UB-for-feature-screening},
i.e.,
$\tilde{UB}(|X_{:j}^\top \alpha^*|) \le UB(|X_{:j}^\top \alpha^*|)$.
\end{theo}

The proof of
Theorem~\ref{theo:Sample2Feature}
is presented in Appendix.
By replacing the upper bounds in
\eq{eq:UB-for-feature-screening}
with
that in
\eq{eq:new-UB-for-feature-screening}
in the safe feature screening step,
there are more chance for screening out non-active features.
%
%We note that the tighter bounds in
%\eq{eq:new-UB-for-feature-screening}
%is different from the upper bound
%that would be obtained just by
%using
%$\tilde{\alpha}$
%as
%the dual feasible solution
%$\hat{\alpha}$
%in
%\eq{eq:UB-for-feature-screening}.

Next,
Theorem~\ref{theo:Feature2Sample}
states that we can obtain tighter lower and upper bounds
for sample screening
by exploiting the result of the previous safe feature screening.
\begin{theo}
 \label{theo:Feature2Sample}
 Consider a safe sample screening problem
 given
 arbitrary pair of primal and dual feasible solutions
 $\hat{w} \in {\rm dom}P_\lambda$
 and
 $\hat{\alpha} \in {\rm dom}D_\lambda$.
 Furthermore,
 suppose that
 the result of the previous safe feature screening step assures that
 $w_j^* = 0$
 for
 a subset of the features
 $j \in \cF \subset [d]$.
 Let
 $\cU_f := [d] \setminus \cF$,
 $r_P := \sqrt{2 G_\lambda(\hat{w}, \hat{\alpha})/\lambda}$,
 and
 $\tilde{w}$
 be a $d$-dimensional vector
 whose element is defined as
 $\tilde{w}_j = \hat{w}_j$
 for
 $j \in \cU_f$
 and
 $\tilde{w}_j = 0$
 for
 $j \in \cF$.
Then,
$x_i^\top w^*$
is bounded from below and above respectively by the following lower and upper bounds:
\begin{subequations} %08:57'ÌŽ®ŒQ
 \label{eq:new-bounds-for-sample-screening}
 \begin{align}
  \label{eq:new-L-for-sample-screening}
  \tilde{LB}(x_i^\top w^*)
  &:=
  x_{i}^\top \tilde{w}
  -
  \|x_{i \cU_f}\|_2
  \sqrt{
  r_P^2 - \|\hat{w}_{\cF}\|_2^2
  }
  \\
  \label{eq:new-UB-for-sample-screening}
  \tilde{UB}(x_i^\top w^*)
  &:=
  x_{i}^\top \tilde{w}
  +
  \|x_{i \cU_f}\|_2
  \sqrt{
  r_P^2 - \|\hat{w}_{\cF}\|_2^2
  }
 \end{align}
\end{subequations}
and
these bounds
in
\eq{eq:new-bounds-for-sample-screening}
are tighter than or equal to
those
in \eq{eq:bounds-for-sample-screening},
i.e.,
$\tilde{LB}(x_i^\top w^*) \ge LB(x_i^\top w^*)$
and
$\tilde{UB}(x_i^\top w^*) \le UB(x_i^\top w^*)$.
\end{theo}

The proof of
Theorem~\ref{theo:Feature2Sample}
is presented in Appendix.
By replacing the lower and the upper bounds in
\eq{eq:bounds-for-sample-screening}
with
those in
\eq{eq:new-bounds-for-sample-screening}
in the safe sample screening step,
there are more chance to be able to screen out non-active samples.
Again,
we note that the tighter bounds in
\eq{eq:new-bounds-for-sample-screening}
are different from the bounds
that would be obtained just by
using
$\tilde{w}$
as
the primal feasible solution
$\hat{w}$
in
\eq{eq:UB-for-feature-screening}.

Theorems~\ref{theo:Sample2Feature} and \ref{theo:Feature2Sample}
suggests that,
by alternately iterating
feature screening and sample screening,
more and more features and samples could be screened out.
This iteration process can be terminated
when there are few chances
to be able to screen out additional features and samples.
Such a termination condition can be developed
by using the results in the next section.

\section{Safe keeping of active features and samples}
\label{sec:safe-keeping}
Safe screening studies
initiated by the seminal work by
% El Ghaoui
\cite{ghaoui2012safe}
enabled us to identify a part of non-active features/samples
before actually solving the optimization problem.
In other words,
safe screening is interpreted as an active set prediction method
without \emph{false negative error}
(an error that truly active features/samples are predicted as non-active).
In this section,
we show that,
by exploiting the two regions in the dual and the primal spaces,
we can develop
an active set prediction method
without \emph{false positive error}
(an error that truly non-active features/samples are predicted as active).
We call the latter approach as
\emph{safe feature/sample keeping}.

Safe feature keeping rule can be constructed by using the region
in the primal space.
Using
$\Theta_{w^*}$
in
\eq{eq:primal_gap_sphere},
we can get the lower bound of
$|w_j^*|$
for
$j \in [d]$
as
$LB(|w_j^*|) := |\hat{w}_j| - \sqrt{2 G_\lambda(\hat{w}, \hat{\alpha})}$.
Using this lower bound,
safe feature keeping rule is simply formulated as the following theorem.
\begin{theo}
 \label{theo:safe-feature-keeping}
 For an arbitrary pair of primal feasible solution
$\hat{w} \in {\rm dom}P_\lambda$
 and dual feasible solution
 $\hat{\alpha} \in {\rm dom}D_\lambda$,
 \begin{align*}
  |\hat{w}_j| - \sqrt{2 G_\lambda(\hat{w}, \hat{\alpha})} > 0
  ~\Rightarrow~
  w_j^* \neq 0 ~ {\text{for}} ~ j \in [d].
 \end{align*}
\end{theo}
%
%\vspace*{-0.5cm}
Similarly,
safe sample keeping rule can be constructed
by using a region
in the dual space.
The condition for
$\alpha_i^*$
being active is written as
$\alpha_i^* \neq -1, 0, 1$.
This can be guaranteed when
the condition
$\alpha^*_i \in (0,1)$
or
$\alpha^*_i \in (-1,0)$ holds
for the $i$-th element of
$\forall \alpha \in \Theta_{\alpha^*}$.
%
% In order to guarantee that
% $\alpha_i^* \neq -1$,
% $\alpha_i^* \neq 0$
% and
% $\alpha_i^* \neq 1$,
% respectively,
% we need
% $LB(\alpha_i^*) > -1$,
% $LB(|\alpha_i^*|) > 0$
% and
% $UB(\alpha_i^*) < 1$.
% %
% Using
% $\Theta_{\alpha^*}$
% in
% \eq{eq:dual_gap_sphere},
% we can get these bounds
% as
% $LB(\alpha_i^*) := \hat{\alpha}_i - \sqrt{2nG_\lambda(\hat{w}, \hat{\alpha})/\gamma}$,
% $LB(|\alpha_i^*|) = |\hat{\alpha}_i| - \sqrt{2 n G_\lambda (\hat{w}, \hat{\alpha})/\gamma}$
% and
% $UB(\alpha_i^*) := \hat{\alpha}_i + \sqrt{2nG_\lambda(\hat{w}, \hat{\alpha})/\gamma}$,
% respectively.
%
% Based on these bounds,
Since $\Theta_{\alpha^*}$ is a sphere,
safe sample keeping rule can be simply derived
by
%arranging
\eq{eq:dual_gap_sphere}.
% safe sample keeping rule is formally described as the following theorem.
\begin{theo}
 \label{theo:safe-sample-keeping}
 For an arbitrary pair of primal feasible solution
$\hat{w} \in {\rm dom}P_\lambda$
 and dual feasible solutionx
 $\hat{\alpha} \in {\rm dom}D_\lambda$,
 \begin{align*}
  &|\hat{\alpha}_i| \!\! - \!\!
  \sqrt{2 n G_\lambda (\hat{w}, \hat{\alpha})/\gamma} > 0
  % \text{ and }
  % \hat{\alpha}_i \!\! - \!\!
  % \sqrt{2nG_\lambda(\hat{w}, \hat{\alpha})/\gamma}\!\! >\!\! - \!1 \\
  % &\text{ and }
  % \hat{\alpha}_i
  % \!\!+ \!\!
  % \sqrt{2nG_\lambda(\hat{w}, \hat{\alpha})/\gamma} < 1
  \text{ and }
  | \hat{\alpha}_i |
  \!\!+ \!\!
  \sqrt{2nG_\lambda(\hat{w}, \hat{\alpha})/\gamma}\! < \! 1
  \Rightarrow
  \alpha_i^* \neq 0, \pm 1
  \text{ for } i \in [n].
 \end{align*}
\end{theo}
%\vspace*{-0.5cm}
When we use safe screening approaches,
there is a trade-off
between
the computational costs of evaluating safe screening rules
and
the computational time saving by screening out some features/samples.
If we know in advance that
some of the features/samples cannot be non-active
by using safe keeping approaches,
we do not have to waste the rule evaluation costs for those features/samples.
Safe screening rule evaluation costs would be more significant
in dynamic screening
and
our simultaneous screening
scenarios
because rules are repeatedly evaluated.
By combining safe screening and safe keeping approaches,
we can get an information about how many features/samples are not yet determined to be active or non-active.
This information can be also used as a stopping criteria of
dynamic screening and our simultaneous screening.
We can stop evaluating safe screening rules
when there only remain few features/samples
that have not been determined to be active or non-active.
Specifically,
if the fraction of the sum of
the safely screened features/samples
and
the safely kept features/samples
is close to one,
then we have few chances to be able to safely screen out additional features/samples.
The additional computational complexities\footnote{
Here, we assume that the duality gap
$G_\lambda(\hat{w}, \hat{\alpha})$
have been already computed.
}
of
safe feature/sample keeping for a single feature/sample
is $\cO(1)$,
which is negligible compared with
$\cO(n)$
complexity for safe feature screening
and
$\cO(d)$
complexity for safe sample screening.

We finally note that,
in our particular working problem of
$L_1$ smooth SVC
and
$L_1$ smooth SVR,
safe keeping is also possible
by using the KKT optimality conditions in
\eq{eq:kkt1_elastic_net}
for features,
and
\eq{eq:kkt2_smoothed_hinge_posi} - \eq{eq:kkt2_soft_insensitive}
for samples.
We describe the details in the Appendix.

\section{LP-based simultaneous safe screening}
\label{sec:LP-based}
In this section,
we consider another empirical risk minimization problem
that induces sparsities
both in features and samples.
Specifically,
we study a problem
with
$L_1$-penalty
$\psi(w) = \|w\|_1$
and
vanilla hinge loss
$\ell_i(w) = \max\{0, 1 - y_i x_i^\top w\}$,
which we call
\emph{LP-based SVM}
because
it is casted into a linear program (LP).
LP-based SVM has been studied in
\cite{bradley1998feature,zhu20041},
and also in boosting context.
LPBoost
\cite{demiriz2002linear}
solves LP-based SVM
via the column generation approach of linear programming.
while
ERLPBoost \cite{warmuth-etal:alt08}
is a variant of LPBoost
obtained by adding a small entropic term in the LP objective.
Sparse LPBoost \cite{hatano-takimoto:ds09}
is simiilar to simultaneous screening
in that it iteratively solves LP sub-problems for features and samples.
LP-based SVM induces
feature sparsity due to $L_1$-penalty
and
sample sparsity due to the property of hinge loss.

The convex conjugate functions of
$L_1$-penalty
and
vanilla hinge loss
are respectively written as
\begin{align}
  \label{eq:l1reg_conj}
  \psi^* \left( v \right) &:= \begin{cases}
    0 & (\| v \|_{\infty} \le 1), \\
    \infty & (\text{otherwise}),
  \end{cases}
  \\
  \label{eq:hinge_conj}
  \ell^*_i(\alpha_i) &:= \begin{cases}
    y_i \alpha_i & y_i \alpha_i \in [-1,0], \\
    \infty & (\text{otherwise}),
  \end{cases}
\end{align}
and the dual problem is written as
\begin{align*}
  \max_{\alpha} D_{\lambda} (\alpha) :=
  \max_{\alpha}
  \left\{ y^\top \alpha \right\}
  ~ {\text{s.t.}} ~
  \left\| \frac{1}{\lambda n} \alpha_i x_i \right\|_{\infty} \le 1,
  ~ y_i \alpha_i \in [0,1] ~~ \forall i \in [n].
\end{align*}

\subsection{Safe feature screening for LP-based SVM}
\label{subsec:safe-feature-screening-LP-SVM}
Feature safe screening for LP-based SVM was studied
in the seminal safe feature screening paper by
\cite{ghaoui2012safe}.
However,
the method presented in their paper
requires a precise optimal solution of an LP-based SVM with a different penalty parameter $\lambda$.
This requirement is impractical
because precise optimal solutions are often difficult to get numerically
as recently pointed out
by
\cite{fercoq2015mind}.
Here,
we present a novel safe feature screening method for LP-based SVM
that only requires an arbitrary pair of
a primal feasible solution
$\hat{w} \in {\rm dom} P_\lambda$
and
a dual feasible solution
$\hat{\alpha} \in {\rm dom} D_\lambda$.
The proposed safe feature screening method for
LP-based SVM
is summarized in the following theorem.
\begin{theo}
 \label{theo:l1l1_feature_screen}
 Consider safe feature screening problem
 given an arbitrary pair of
 a primal feasible solution
 $\hat{w} \in {\rm dom}P_\lambda$
 and
 a dual feasible solution
 $\hat{\alpha} \in {\rm dom}D_\lambda$.
 Let
 $\ell_q := \lfloor y^\top \hat{\alpha} \rfloor$,
 $u_q = \lfloor P_{\lambda}(\hat{w})  \rfloor$,
 $Z := [y_1 x_1, \ldots, y_n x_n]^\top \in \RR^{n \times d}$,
 and
 $Z_{:j}^\prime \in \RR^n$,
 $j \in [d]$,
 be the vector
 obtained by sorting
 $Z_{:j}$
 in increasing order.
 Furthermore,
 let
 $n_{Z_{:j}^\prime}$
 and
 $p_{Z_{:j}^\prime}$
 represent
 the numbers of negative and positive elements of
 $Z_{:j}^\prime$,
 respectively.
 Then,
 % \vspace*{-0.3cm}
 \begin{align*}
  LB(X_{:j}^\top \alpha^*) < - \lambda n
  ~\text{and}~
  UB(X_{:j}^\top \alpha^*) > \lambda n
  ~\Rightarrow~
  w_j^* = 0,
 \end{align*}
 % \vspace*{-0.3cm}
 where
   \begin{align*}
    LB (X_{:j}^\top \alpha^* ) := \begin{cases}
    \sum_{i=1}^{l_q}
    Z'_{ij} + (y^\top \hat{\alpha} - l_q) Z'_{(l_q+1) j}
    & (n_{Z'_{:j}} < l_q + 1), \\
    \sum_{i=1}^{u_q} Z'_{ij} +
    (P_{\lambda}(\hat{w}) - u_q),
     Z'_{u_q j}
    & (n_{Z'_{:j}} > u_q ) \\
    \sum_{i=1}^{n} \min \{0, Z'_{ij} \}
    & ({\rm otherwise}),
    \end{cases}
    \end{align*}
    \begin{align*}
    UB (X_{:j}^\top \alpha^* ) :=
    \begin{cases}
    \!\sum_{i=n - l_q}^{n}
    \!\!Z'_{ij} + (y^\top \hat{\alpha} - l_q) Z'_{(n-l_q-1) j}
    &\!\!\!\!\!\! (p_{Z'_{:j}} < l_q + 1), \\
    \!\sum_{i=n-u_q}^{n} \!\! Z'_{ij} +
    (P_{\lambda}(\hat{w}) - u_q)
     Z'_{(n - u_q -1) j}
    &\!\!\!\!\!\! (p_{Z'_{:j}} > u_q ), \\
    \!\sum_{i=1}^{n} \max \{0, Z'_{ij} \}
    &\!\!\!\!\!\! ({\rm otherwise}).
    \end{cases}
  \end{align*}
\end{theo}
\vspace*{-0.4cm}
The proof of
Theorem~\ref{theo:l1l1_feature_screen}
is presented in Appendix.

\subsection{Safe sample screening for LP-based SVM}
\label{subsec:safe-sample-screening-LP-SVM}
Here, we develop a novel safe sample screening method for LP-based SVM as summarized in the following theorem.
\begin{theo}
 \label{theo:l1l1_sample_screen}
 Consider safe sample screening
 given an arbitrary primal feasible solution
 $\hat{w} \in {\rm dom}P_\lambda$.
 Let
 $g_{\ell_i}(w)$
 be a subgradient of
 vanilla hinge loss
 $\ell_i(w)$
 for
 $w \in {\rm dom}P_\lambda$,
 and define
 $k := \lambda \|\hat{w}\|_1 + \frac{1}{n} \sum_{i \in [n]} g_{\ell_i(\hat{w})}^\top \hat{w}$.
 Then,
 \begin{align*}
  &
  y_i = +1
  \text{ and }
  LB(x_i^\top w^*) > +1
  ~\Rightarrow~
  \alpha_i^* = +1,
  \\
  &
  y_i = +1
  \text{ and }
  UB(x_i^\top w^*) < +1
  ~\Rightarrow~
  \alpha_i^* = 0,
  \\
  &
  y_i = -1
  \text{ and }
  LB(x_i^\top w^*) > -1
  ~\Rightarrow~
  \alpha_i^* = 0,
  \\
  &
  y_i = -1
  \text{ and }
  UB(x_i^\top w^*) < -1
  ~\Rightarrow~
  \alpha_i^* = -1,
 \end{align*}
 where
 \begin{subequations}
  % \label{eq:bounds-l1l1-sample-screening}
 \begin{align*}
  \!\!&LB(x_i^\top\! w^*) \!\!:=\!\! \max_{\mu > 0} \{\mu k \}
  ~
  {\rm{s.t.}}
  \left\|\!\! - \frac{1}{\lambda} x_i -
  \frac{\mu}{\lambda n} \sum_{i=1}^{n} g_{\ell_i}(\hat{w}) \right\|_{\infty}
  \le \mu, \\
  \!\!&UB(x_i^\top\! w^*)\!\!:=\!\! \max_{\mu > 0} \{\mu k \}
    ~
  {\rm{s.t.}}
  \left\|\!\!\!\! \phantom{-} \frac{1}{\lambda} x_i -
  \frac{\mu}{\lambda n} \sum_{i=1}^{n} g_{\ell_i}(\hat{w}) \right\|_{\infty}
  \le \mu.
 \end{align*}
 \end{subequations}
\end{theo}
The proof of
Theorem~\ref{theo:l1l1_sample_screen}
is presented in Appendix.
Although the lower and the upper bounds
% in \eq{eq:bounds-l1l1-sample-screening}
are not explicitly presented,
these optimization problems can be easily solved
because they are just linear programs with one variable $\mu > 0$.
\begin{figure}[!h]
  \begin{center}
   \includegraphics[clip,width=16cm]{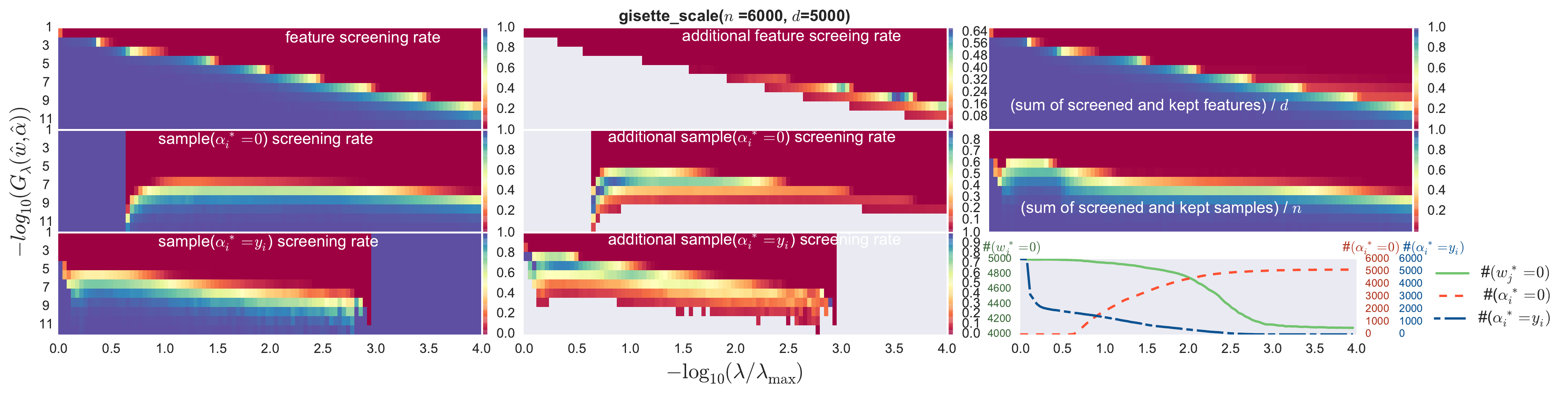}\\
   \includegraphics[clip,width=16cm]{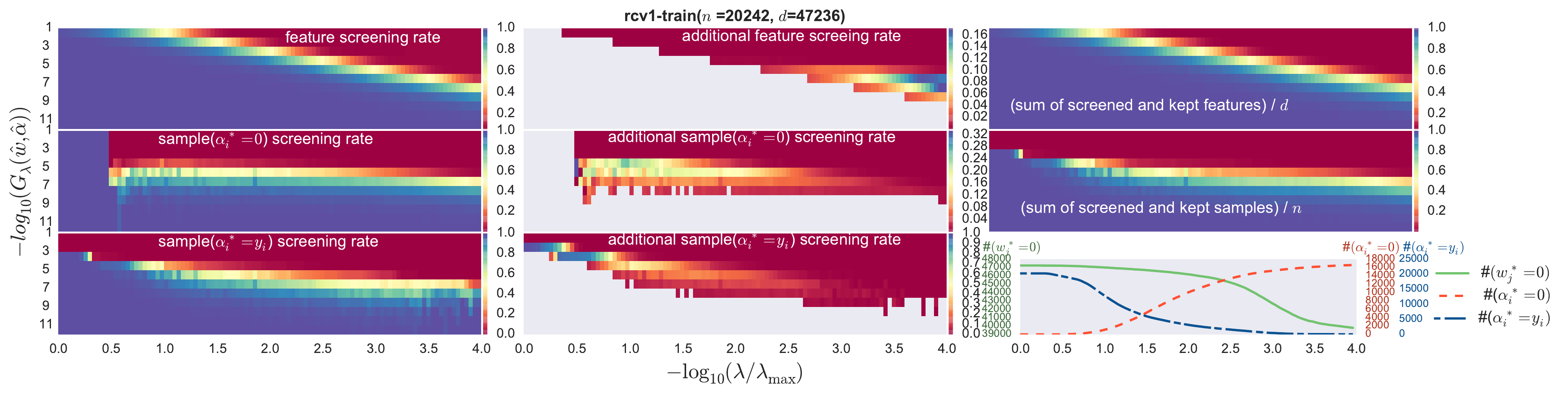}\\
   \includegraphics[clip,width=16cm]{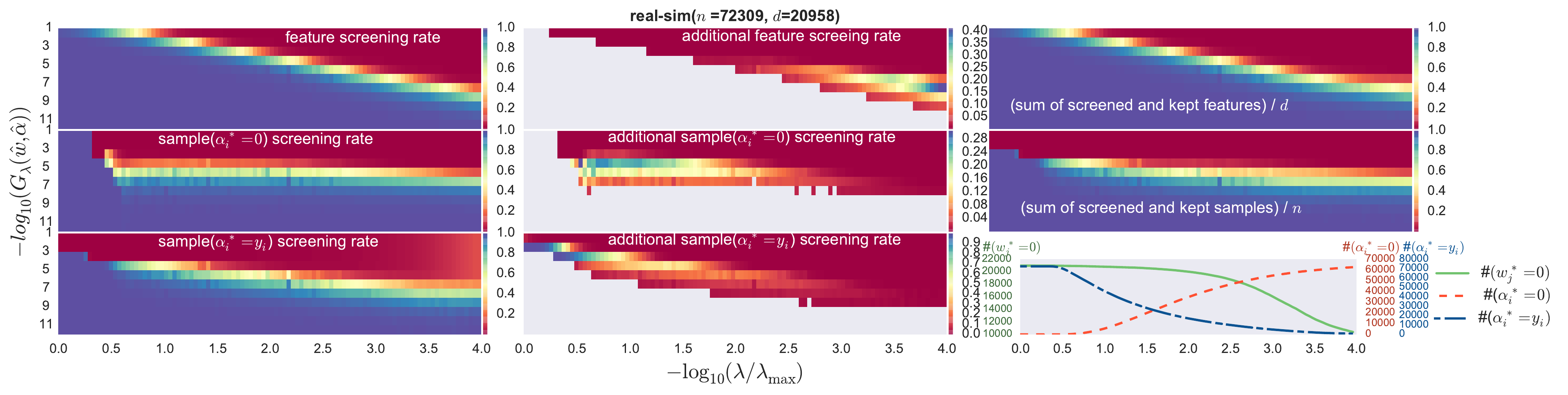}
   \includegraphics[clip,width=16cm]{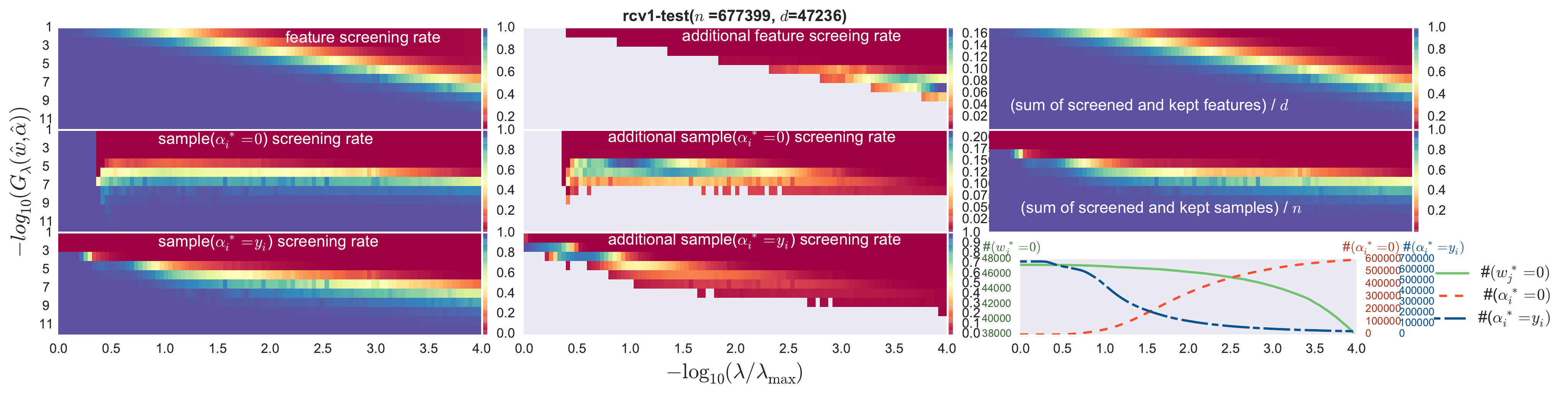}
   \caption{Safe screening and keeping rates for classification problems
   (for {\tt real-sim} and {\tt rcv1-test} datasets).
   The three plots
   in the left show
   the individual safe feature/sample screening rates (the middle and the bottom ones are for $\alpha_i^* = 0$ and $\alpha_i^* = \pm 1$, respectively).
   The three plots
   in the center show
   the additional safe screening rates by simultaneously considering feature and sample screenings.
   The gray area
   in these center plots
   corresponds to the blue area
   in the corresponding left plot.
   In these gray area, the individual safe screening performances are good enough (screening rate $> 0.95$) and additional screening is unnecessary.
   The top right and middle right plots show the safe keeping rates for feature and samples, respectively.
   The bottom right plot shows the numbers of active features and samples for various values of $\lambda$.}
  \label{fig:screeing_rate_smoothed_hinge}
  \end{center}
\end{figure}
As we discussed in \S\ref{sec:simultaneous},
by alternatively iterating
safe feature screening in
\S\ref{subsec:safe-feature-screening-LP-SVM}
and
safe sample screening in
\S\ref{subsec:safe-sample-screening-LP-SVM},
we can make the regions in the dual and the primal regions step by step,
indicating that the chance of screening out more features and samples increases~\footnote{
The tighter bounds
can be obtained
by exploiting the previous safe screening results
as we discussed in \S\ref{sec:simultaneous}
although we do not explicitly present those bounds here
due to the space limitation.
}.

\section{Numerical experiments}
\label{sec:experiments}
%In this section,
We demonstrate the advantage of
simultaneous safe screening
through numerical experiments.
After we describe the experimental setups
in
\S\ref{subsec:experimental-setup},
we report the results on
safe screening and keeping rates,
and
computation time savings
in
\S\ref{subsec:safe-screening-keeping-rate}
and
\S\ref{subsec:computation-time-savings},
respectively.
%
% Due to the space limitation,
% we only show the results on classification problems in the main text.
%
% Other experimental results are presented in Appendix~\ref{app:other-experiments}.

\begin{figure}[t]
  \begin{center}
   \includegraphics[clip,width=17cm]{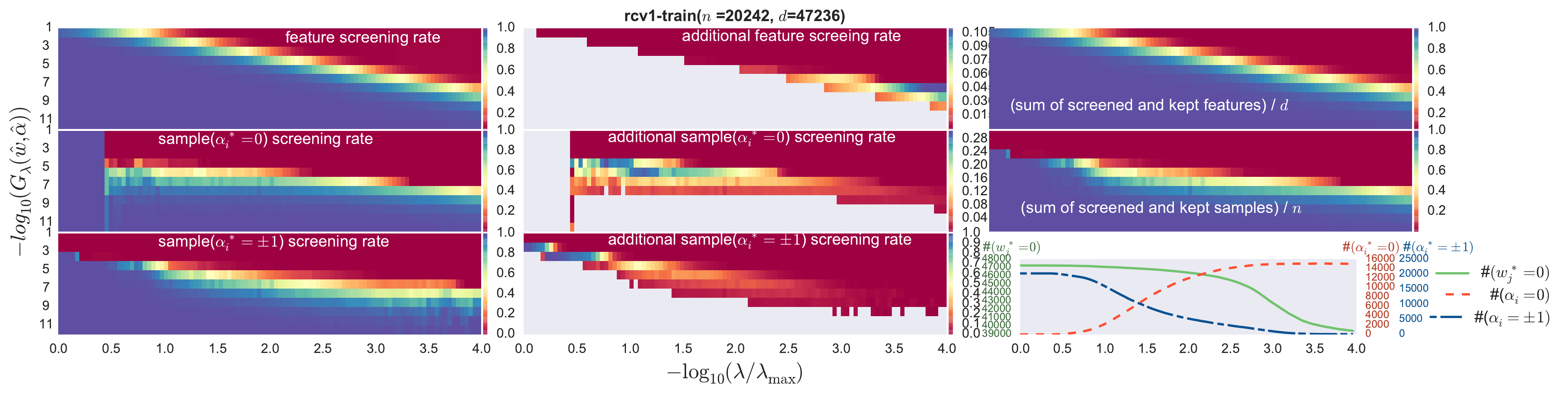}
   \includegraphics[clip,width=17cm]{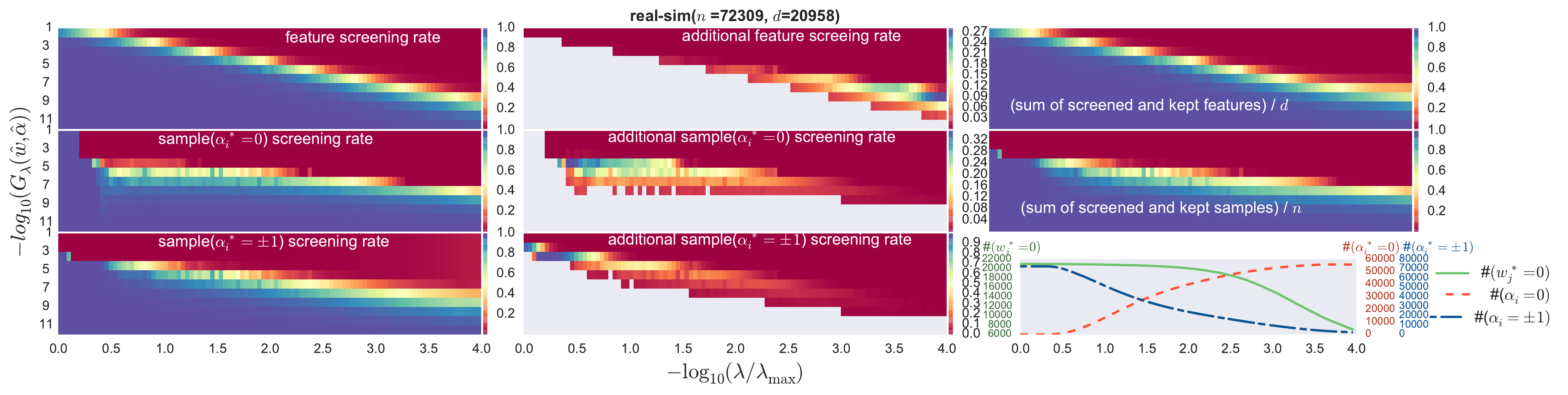}
   % \subfigure{\includegraphics[clip,width=18cm]{./Fig/regression/0203_rate2_E2006-train_20000}}
   \caption{Safe screening and keeping rates in regression problems. See the caption in \figurename~\ref{fig:screeing_rate_smoothed_hinge}.}
   \label{fig:screeing_rate_smoothed_insensitive}
  \end{center}
\end{figure}

\begin{figure}[t]
  \begin{center}
   \subfigure{\includegraphics[clip,width=4cm]{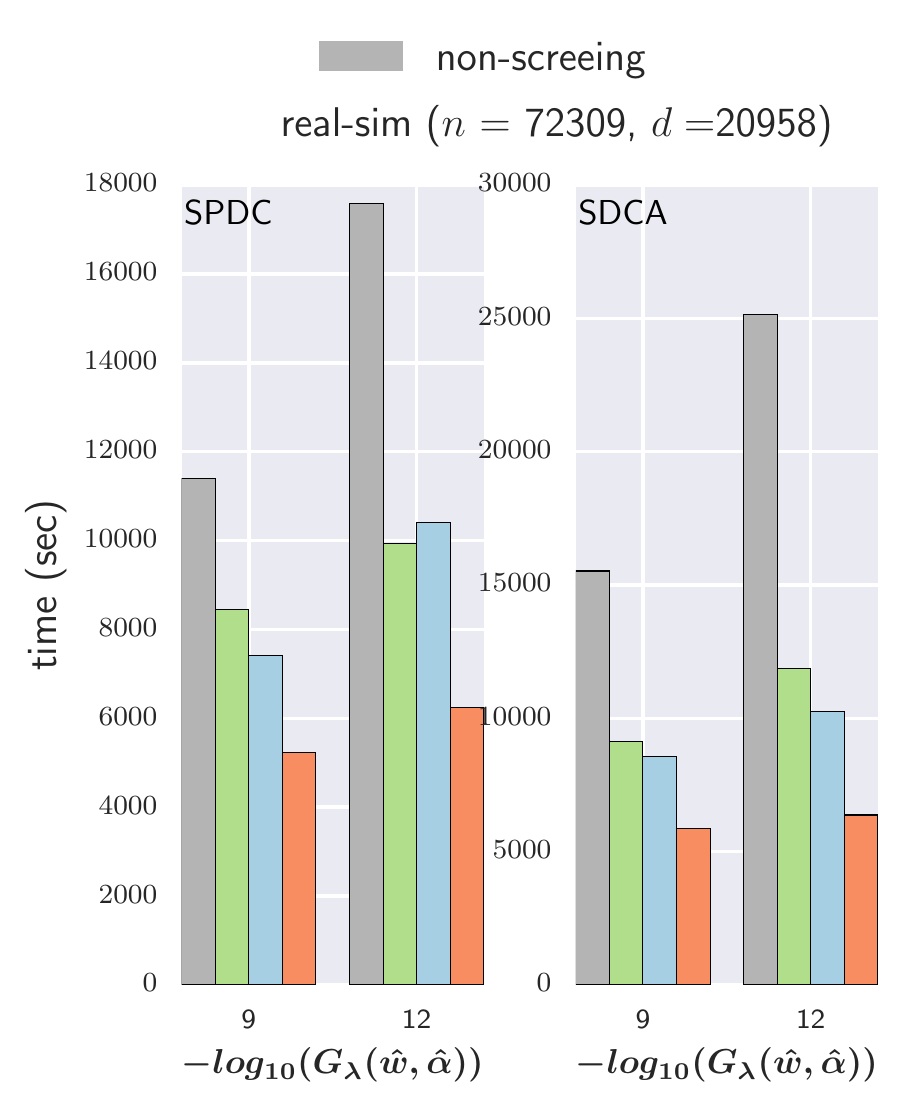}}
   \subfigure{\includegraphics[clip,width=3.8cm]{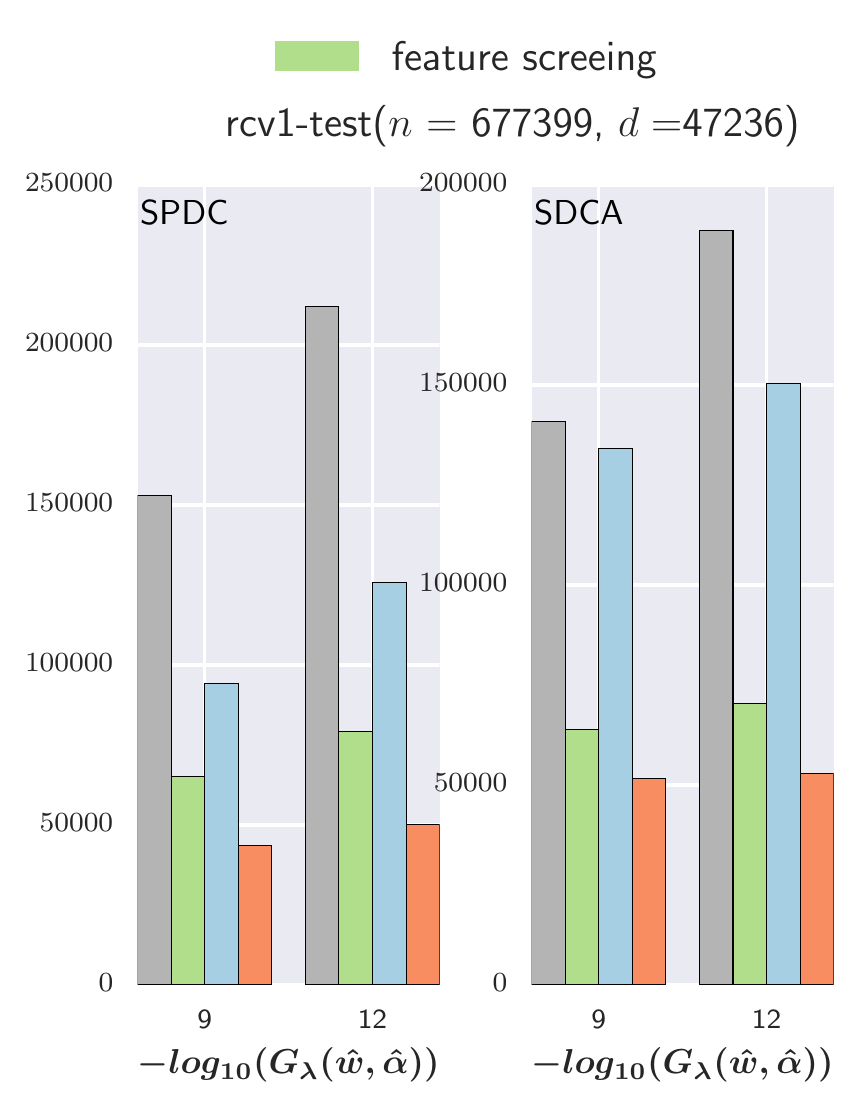}}
   \subfigure{\includegraphics[clip,width=3.8cm]{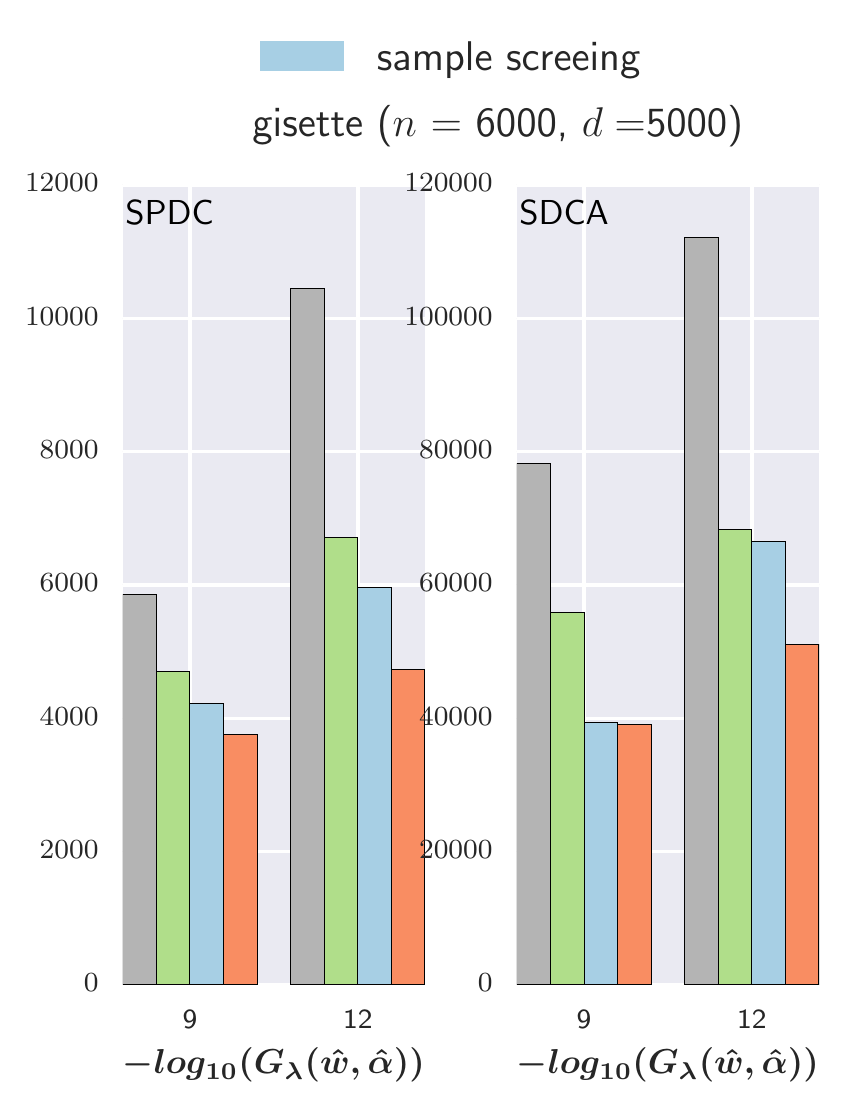}}
   \subfigure{\includegraphics[clip,width=3.85cm]{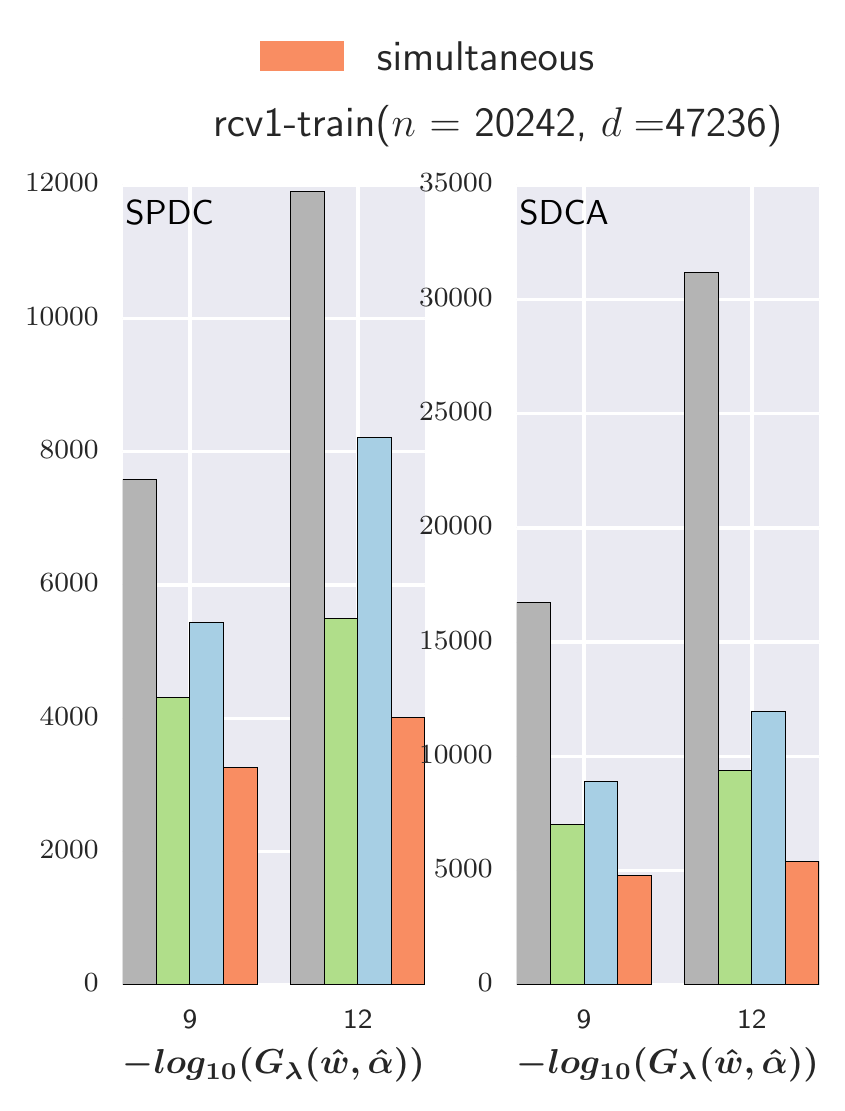}}
   \caption{Total computation time for training 100 solutions for various values of $\lambda$ in classification problems.}
    \label{fig:acc_time_classification}
  \end{center}
\end{figure}

\begin{figure}[t]
  \begin{center}
   \subfigure{\includegraphics[clip,width=4.0cm]{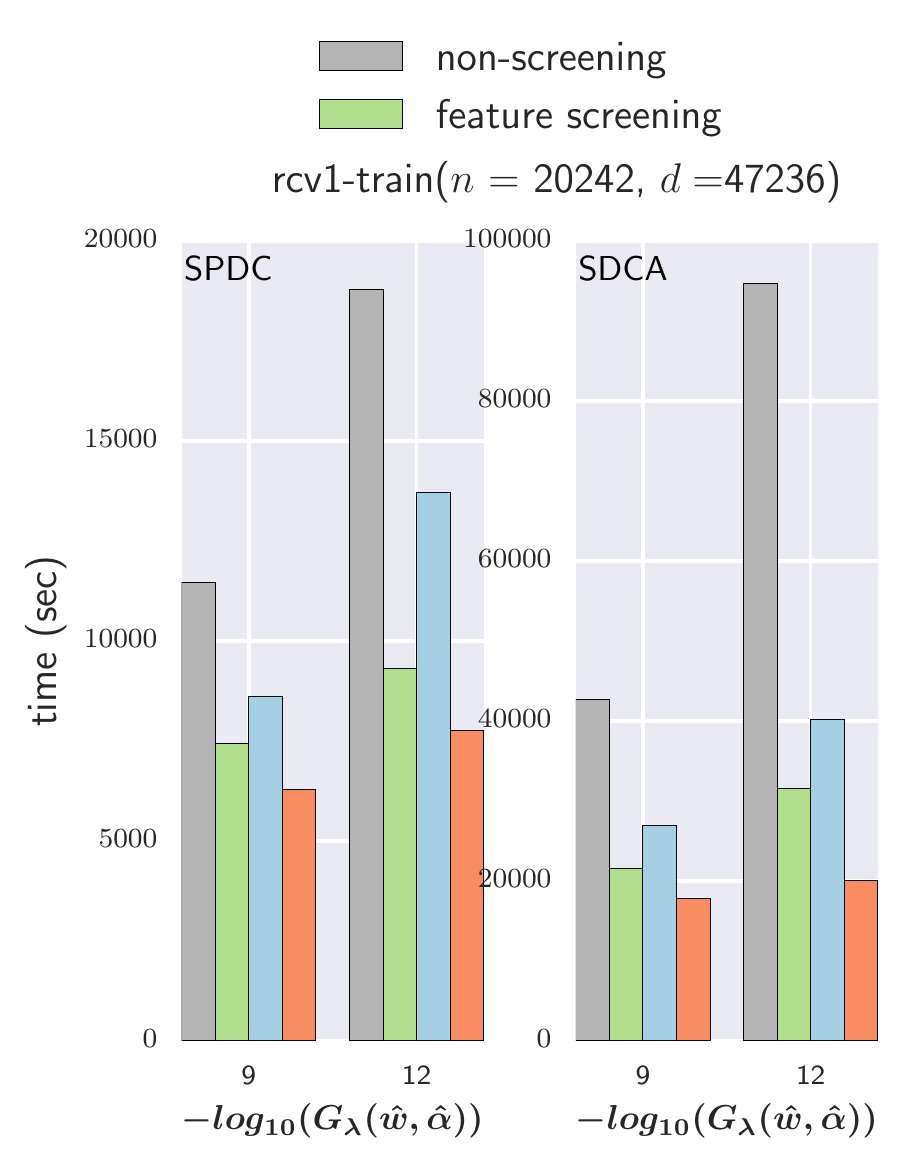}}
   \subfigure{\includegraphics[clip,width=4.0cm]{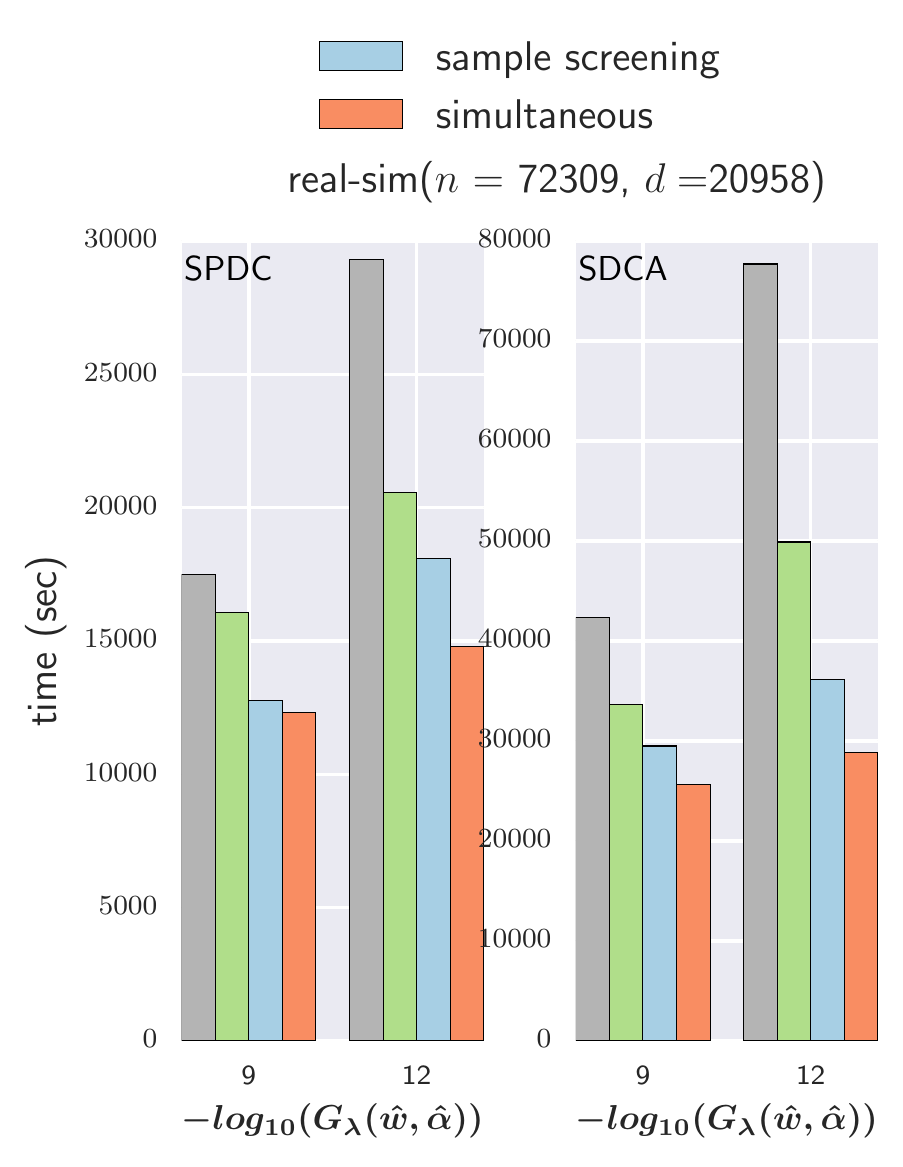}}
   % \subfigure{\includegraphics[clip,width=4.0cm]{./Fig/regression/bar_E2006-train_20000}}
   \caption{Total computation time for training 100 solutions for various values of $\lambda$ in regression problems.}
    \label{fig:acc_time_regression}
  \end{center}
\end{figure}

\subsection{Experimental setups}
\label{subsec:experimental-setup}
Table \ref{tab:datasets} summarizes the datasets used in the experiments.
We picked up four datasets whose numbers of features and samples are both large from libsvm dataset repository
\cite{chang2011libsvm}.
These are all classification datasets. For regression experiments, we just used the label indicator variables as scalar response variables.

\begin{table}[h]
 \vspace*{-5mm}
 \begin{center}
  \caption{Benchmark datasets used in the experiments. }
  \label{tab:datasets}
  \vspace*{1mm}
  \begin{scriptsize}
  \begin{tabular}{c|r|r|r}
   dataset name& sample size: $n$  & feature size: $d$ & \#(nnz)/$nd$ \\ \hline\hline
   {\tt real-sim}    & 72,309  & 20,958 & 0.002448 \\ \hline
   {\tt rcv1-test}   & 677,399 & 47,236 & 0.025639 \\ \hline
   {\tt gisette}     & 6,000   & 5,000  & 0.991000 \\ \hline
   {\tt rcv1-train}  & 20,242  & 47,236 & 0.001568
   % \\ \hline
   % {\tt E2006-tfidf} & 16,087  & 20,000 & 0
  \end{tabular}
   \\
   \#(nnz) indicates the number of non-zero elements.
  \end{scriptsize}
 \end{center}
 \vspace*{-5mm}
\end{table}

Here,
we report the results on
$L_1$-penalized smoothed hinge SV classification
and
$L_1$-penalized smoothed $\veps$-insensitive SV regression.
We set
$\lambda_{\max} := \| Z^\top \bm{1} \|_{\infty} $
for classification and
$\lambda_{\max} := \| X^\top \bm{1} \|_{\infty} $
for regression,
and considered problems with various values of the penalty parameter $\lambda$
between
$\lambda_{\max}$
and
$10^{-4} \lambda_{\max}$.
The parameter in the smoothed hinge loss $\gamma$ is set to be $0.5$.
Also, the parameters in the smoothed $\veps$-insensitive loss
$\gamma$ is set to be $0.1$ and $\veps$ is set to be $0.5$.
The proposed methods can be used with any optimization solvers
as long as
they provide both primal and dual sequences of solutions that converge to the optimal solution.
For the experiments,
we used
Proximal Stochastic Dual Coordinate Ascent (SDCA) \cite{shalev2015accelerated}
and
Stochastic Primal-Dual Coordinate (SPDC) \cite{Zhang2015spdc}
because they are state-of-the-art optimization methods
for general large-scale regularized empirical risk minimization problems.
We wrote the code in C++
along with Eigen library for some numerical computations.
The code is provided as a supplementary material, and will be put in public domain after the paper is accepted.
All the computations were conducted by using a single core of an Intel Xeon CPU E5-2643 v2 (3.50GHz), 64GB MEM.

\subsection{Safe screening and keeping rates}
\label{subsec:safe-screening-keeping-rate}
For demonstrating the synergy effect of simultaneous safe screening,
we compared the simultaneous screening rates with individual safe screening rates.
%
%Note that screening rates are independent of the choice of the optimization solvers.
%
{\figurename}s
\ref{fig:screeing_rate_smoothed_hinge}
and
\ref{fig:screeing_rate_smoothed_insensitive}
shows the results
on classification and regression problems,respectively.
In the $3 \times 3$ subplots,
the letf plots indicate
the individual screening rates defined as
(\#(screened features or samples) / \#($w^*_j  = 0$ or $\alpha^*_i = 0$ or $\alpha^*_i = \pm 1$)).
The center plots indicate the additional screening rates by the synergy effect
defined as
(\#(additionally screened features or samples) / \#($w^*_j  = 0$ or $\alpha^*_i = 0$ or $\alpha^*_i = \pm 1$)).
The top plots represent the results on feature screening,
while the middle and the bottom plots
show the results on sample screening
(for each of $\alpha^*_i = 0$ and $\alpha_i^*=\pm 1$).
We investigated the screening rates
for various values of $\lambda$
in the horizontal axis
and
for various quality of the solutions measured in terms of the duality gap in the vertical axis.
In all the datasets,
we observed that it is valuable to consider both feature and sample screening
when the numbers of features and samples are both large.
In addition,
we confirmed that there are improvements in screening rates by the synergy effect both in feature and sample screenings
especially when duality gap $G_{\lambda}(\hat{w}, \hat{\alpha}) $ is large.
Note that gray areas in the center plots corresponds to the blue area in the corresponding left plot,
where the individual safe screening performances are good enough (screening rate $> 0.95$) and additional screening is unnecessary.

The top left and the middle left plots show
the rates of features and samples,
respectively,
that are determined to be active or non-active
by using safe keeping and safe screening approaches,
respectively.
We see that,
by combining safe keeping and safe screening approaches,
a large portion of features/samples can be determined to be active or non-active
without actually solving the optimization problems.
The bottom right plot shows the number of non-active features/samples for various values of $\lambda$.
\label{subsec:computation-time-savings}

\subsection{Computation time savings}
We compared
the computational costs of
simultaneous safe screening
and individual safe feature/sample screening
with the naive baseline (denoted as ``non-screening'').
We compared the computation costs in a realistic model building scenario.
Specifically,
we computed a sequence of solutions at $100$ different penalty parameter values
evenly allocated in
$[10^{-4} \lambda_{\max}, \lambda_{\max} ] $
in the logarithmic scale.
In all the cases,
we used {\it warm-start} approach, i.e.,
when we computed a solution at a new $\lambda$,
we used the solution at the previous $\lambda$ as the initial starting point of the optimizer.
In addition,
whenever possible,
we used
{\it dynamic safe screening} strategies \cite{bonnefoy2014dynamic}
in which safe screening rules are evaluated
every time the duality gap
$G_{\lambda}(\hat{w}, \hat{\alpha})$
was
$0.1$ times smaller than before.
Here,
we exploited the information obtained by safe keeping as well,
i.e.,
we did not evaluate safe screening rules for features and samples which are safely kept as active,
and
the rate of features/samples that are determined to be active or non-active
(see the left top and left middle plots in {\figurename} \ref{fig:screeing_rate_smoothed_hinge})
is used as the stopping criterion for safe feature rule evaluations.
(we terminated safe screening rule evaluations when the rate reaches $0.95$).

{\figurename}s
\ref{fig:acc_time_classification}
and
\ref{fig:acc_time_regression}
show the entire computation time for training $100$ different solutions.
In all the datasets,
simultaneous safe screening was significantly faster than individual safe feature/sample screening and non-screening.
Figure
\ref{fig:per_itr_time_classification_spdc_rcv1_test}
shows a sequence of computation times for various values of $\lambda$
for the classification problem on {\tt rcv1-test} and {\tt real-sim} datasets with SPDC optimization solver.
%(similar plots can be obtained for other settings).
%
These plots suggest that
the
computation time savings by individual safe feature screening was better than individual safe sample screening when $\lambda$ is large
because
the feature screening rates are high when $\lambda$ is large,
while
the difference between the two individual screening approaches gets smaller
as $\lambda$ gets smaller
(see \figurename~\ref{fig:screeing_rate_smoothed_hinge}).
Simultaneous safe screening
was
consistently
faster than individual safe feature/sample screening and non-screening
in all the problem setups.

\begin{figure}[!h]
  \begin{center}
   \subfigure{\includegraphics[clip,width=7.0cm]{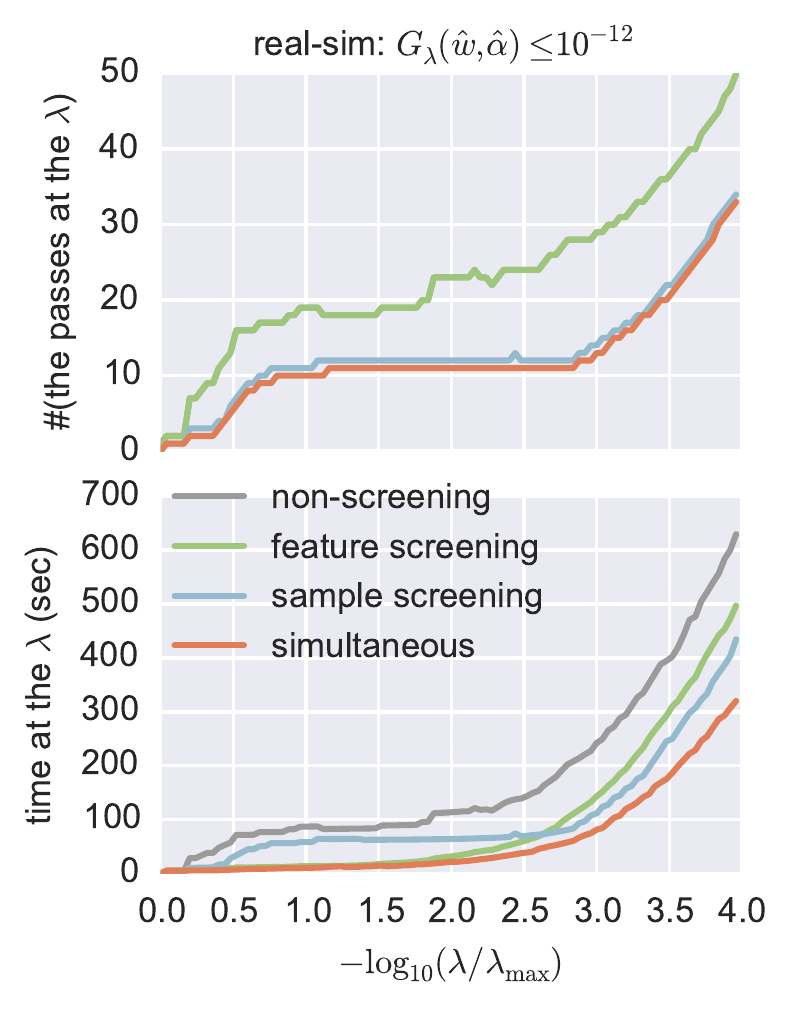}}
   \subfigure{\includegraphics[clip,width=7.0cm]{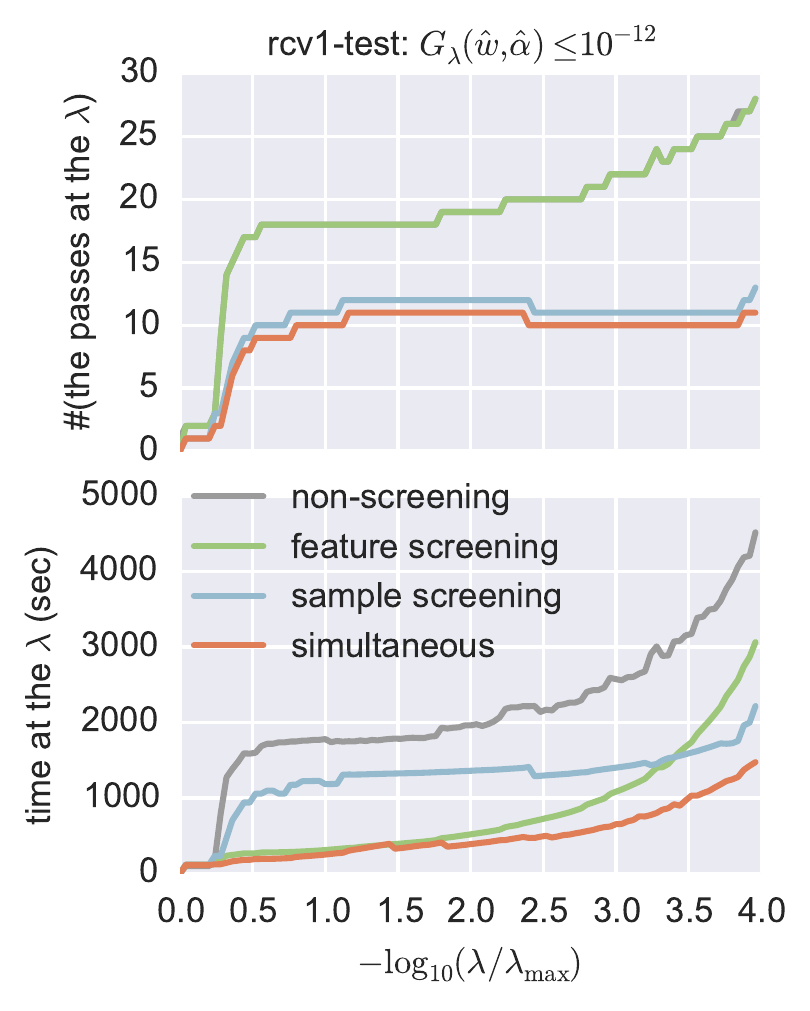}}
   \subfigure{\includegraphics[clip,width=7.0cm]{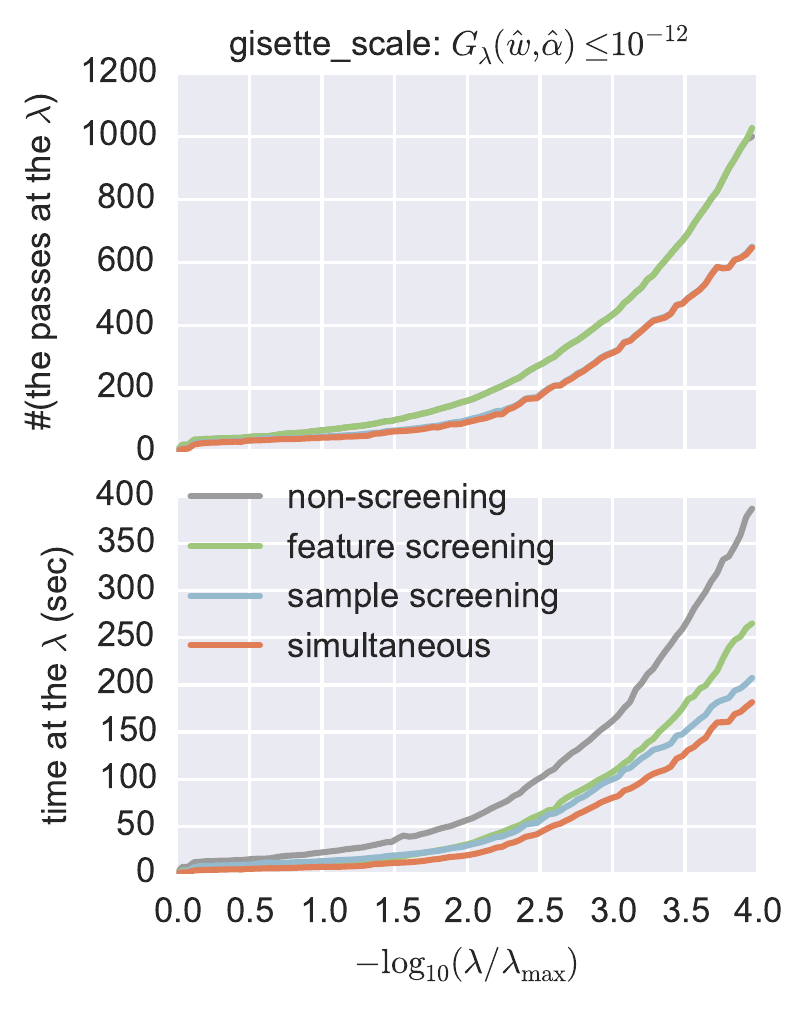}}
   \subfigure{\includegraphics[clip,width=7.0cm]{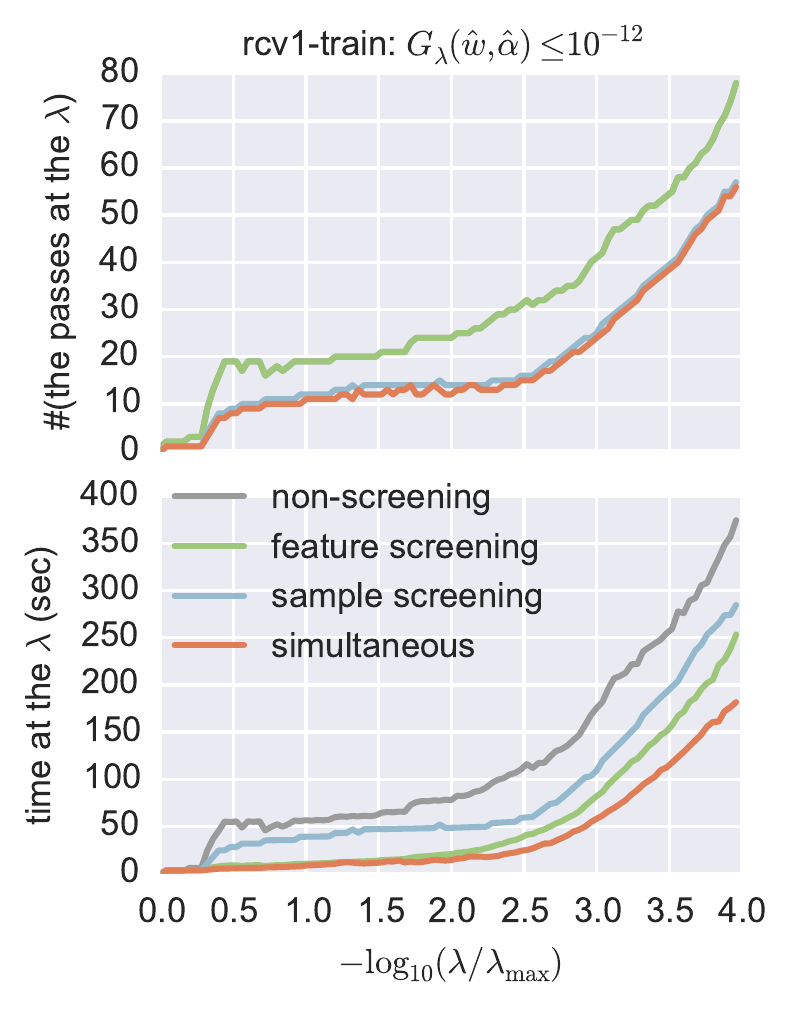}}
   \caption{
   Number of optimization steps and computation time.
   % (for {\tt real-sim} and {\tt rcv1-test} datasets).
   Sequences of the number of passes through the entire dataset and
   computation time
   to convergence
   for various values of $\lambda$ for classification problems with SPDC solver are plotted.}
   \label{fig:per_itr_time_classification_spdc_rcv1_test}
  \end{center}
\end{figure}

 \vspace*{-0.3cm}
\section{Conclusions}
\label{sec:conclusions}
We introduced a new approach
for safely screening out
features
and
samples
simultaneously.
We showed that alternatively iterating feature and sample screening steps has synergetic advantages
in that screening performances can be radily improved in steps.
We also introduced a new approach for predicting active set without false positive error,
which we called
\emph{safe keeping}.
Intensive numerical experiments demonstrated the advantage of our approaches
in classification and regression problems
with large numbers of features and samples.

\clearpage

\appendix

\section{Proofs}
\label{app:proofs}

\subsection{Proof of Lemma 1}
\label{subsec:proof_of_Lemma1}
\begin{proof}
Since
$P_{\lambda}(w) $
is $\lambda$-strongly convex,
$\forall w_1, w_2 \in {\rm dom} P_{\lambda}$,
\begin{align*}
  & P_{\lambda} (w_1) \ge P_{\lambda} (w_2) +  g_{P_{\lambda}}(w_2)^\top (w_1 - w_2) + \frac{\lambda}{2} \| w_1 - w_2 \|^2_2,
\end{align*}
where, $g_{P_{\lambda}}(w) \in \partial P_{\lambda} (w) $.
On the other hand,
$\forall \hat{w} \in {\rm dom} P_{\lambda} ,$
$ g_{P_{\lambda}}(w^* )^\top (\hat{w} - w^*) \ge 0$
(see Proposition B.24 in \cite{Bertsekas99a}).
Also, from weak duality,
$\forall \hat{\alpha} \in {\rm dom} D_{\lambda},$
$ D(\hat{\alpha}) \le P_{\lambda}(w^*)$.
By substituting $w_1 = \hat{w}, w_2 = w^*$,
\begin{align*}
  & \frac{\lambda}{2} \| \hat{w} - w^* \|^2_2 \le P_{\lambda} (\hat{w}) - D_{\lambda} (\hat{\alpha}).
\end{align*}
Therefore,
$w^*$
is
within a
% compact
region $\Theta_{w^*}$,
where
\begin{align*}
  % \label{eq:primal_gap_sphere}
  \Theta_{w^*} :=
  \set{ w |
  \| \hat{w} - w \|_2 \le
  \sqrt{2 G_{\lambda}(\hat{w}, \hat{\alpha}) / \lambda} } .
  % \\
  % ~~ \forall \hat{w} \in {\rm dom} P_{\lambda},
  % ~ \forall \hat{\alpha} \in {\rm dom} D_{\lambda}.
\end{align*}

Since $\Theta_{w^*}$ is Sphere,
a lower bound of $x_i^\top w^*$ and
an upper bound of $x_i^\top w^*$ are
given in closed form as follows:
\begin{align*}
LB(x_i^\top w^*) &= x_i^\top \hat{w} -
\|x_i\|_2 \sqrt{2 G_{\lambda}(\hat{w}, \hat{\alpha}) / \lambda} , \\
UB(x_i^\top w^*) &= x_i^\top \hat{w} +
\|x_i\|_2 \sqrt{2 G_{\lambda}(\hat{w}, \hat{\alpha}) / \lambda } .
\end{align*}
\end{proof}

\subsection{Proof of Theorem 2}
\label{subsec:proof_of_Theorem2}
\begin{proof}
 Supposing that
 the result of the previous safe sample screening step assures the optimal values
 $\alpha_i^*$
 for
 a subset of the samples
 $i \in \cS \subset [n]$,
 the
 % compact
 the dual optimal solution region
 is written as
 \begin{align*}
  \tilde{\Theta}_{\alpha^*} := \set{ \alpha \in \Theta_{\alpha^*} | \alpha_i = \alpha^*_i ~ \forall i \in \cS }.
 \end{align*}
Then,
$X_{:j}^\top \alpha^*$
is bounded from above by the following upper bound:
\begin{align*}
  \tilde{UB}(X_{:j}^\top \alpha^*  )
   &:= \max_{\alpha \in \tilde{\Theta}_{\alpha^*}} X_{:j}^\top \alpha  \\
  &= \sum_{i \in \cS} \alpha^*_i X_{ij}  +
   \max_{\alpha_{\cU_s} }   X_{:j \cU_s} ^\top \alpha_{\cU_s}
  {\rm  ~~s.t ~ }  \| \hat{\alpha}_{\cU_s} - \alpha_{\cU_s} \|^2_2 \le r_D^2 -\| \hat{\alpha}_{\cS} - \alpha^*_{\cS} \|^2_2 \\
  &=\! \sum_{i \in \cS} \!\! \alpha^*_i X_{ij} \! + \!
    X_{:j \cU_s} ^\top \hat{\alpha}_{\cU_s} \! + \!
    \|X_{:j \cU_s}\|_2 \sqrt{r_D^2 -\| \hat{\alpha}_{\cS} - \alpha^*_{\cS} \|^2_2} \\
  &= X_{:j} ^\top \tilde{\alpha}  + \|X_{:j \cU_s}\|_2 \sqrt{r_D^2 -\| \hat{\alpha}_{\cS} - \alpha^*_{\cS} \|^2_2}.
\end{align*}
Similarly, $X_{:j}^\top \alpha^*$ is bounded from below by the following lower bound:
\begin{align*}
 \!\! \tilde{LB}(X_{:j}^\top \! \alpha ) \!\!
 % \!\! := \min_{\alpha \in \Theta_{\alpha^*}}
   := X_{:j} ^\top \tilde{\alpha}  - \|X_{:j \cU_s}\|_2 \sqrt{r_D^2 -\| \hat{\alpha}_{\cS} - \alpha^*_{\cS} \|^2_2}.
\end{align*}
Therefore,
\begin{align*}
  \tilde{UB}(| X_{:j}^\top \alpha | )
  = | X_{:j} ^\top \tilde{\alpha} | + \|X_{:j \cU_s}\|_2 \sqrt{r_D^2 -\| \hat{\alpha}_{\cS} - \alpha^*_{\cS} \|^2_2}.
\end{align*}

Since $\tilde{\Theta}_{\alpha^*} \subset {\Theta}_{\alpha^*}$,
the upper bound
% $\tilde{UB}(|X_{:j}^\top \alpha^*|)$
in \eq{eq:new-UB-for-feature-screening}
is tighter than or equal to
that
in \eq{eq:UB-for-feature-screening},
i.e.,
$\tilde{UB}(|X_{:j}^\top \alpha^*|) \le UB(|X_{:j}^\top \alpha^*|)$.
\end{proof}

\subsection{Proof of Theorem 3}
\label{subsec:proof_of_Theorem3}
\begin{proof}
 Supposing that
 the result of the previous safe feature screening step assures that
 $w_j^* = 0$
 for
 a subset of the features
 $j \in \cF \subset [d]$,
 the
 % compact
 the primal optimal solution region
 is written as
 \begin{align*}
  \tilde{\Theta}_{w^*} := \set{w \in \Theta_{w^*} | w_j = 0 ~\forall j \in \cF }.
 \end{align*}
 Then,
 $x_i^\top w^*$
 is bounded from below by the following lower bound:
  \begin{align*}
   \tilde{LB}(x_i^\top w)
   &:=
   \min_{ w \in \tilde{\Theta}_{w^*} } x_i^\top w \\
   &= \min_{ w } x_i^\top w
    {\rm ~~s.t.~~} \| \hat{w} - w \|_2^2 \le r_P^2 ,
    \hat{w}_j = 0 ~ \forall j \in \cF \\
    & = \min_{ w } x_{i \cU_f}^\top w_{\cU_f}
     {\rm ~~s.t.~~} \| \hat{w}_{\cU_f} - w_{\cU_f} \|_2^2
     \le r_P^2 - \| \hat{w}_\cF \|_2^2 \\
    &= x_{i \cU_f}^\top \hat{w}_{\cU_f} -
    \|x_{i \cU_f}\|_2 \sqrt{ r_P^2 - \| \hat{w}_\cF \|_2^2}.
  \end{align*}

Similarly,
 $x_i^\top w^*$
 is bounded from above by the following upper bound:
\begin{align*}
UB(x_i^\top w^*)  = x_{i \cU_f}^\top \hat{w}_{\cU_f} +
    \|x_{i \cU_f}\|_2 \sqrt{ r_P^2 - \| \hat{w}_\cF \|_2^2}.
\end{align*}

Since
$\tilde{\Theta}_{w^*} \subset {\Theta}_{w^*}$,
these bounds
in
\eq{eq:new-bounds-for-sample-screening}
are tighter than or equal to
those
in \eq{eq:bounds-for-sample-screening},
i.e.,
$\tilde{LB}(x_i^\top w^*) \ge LB(x_i^\top w^*)$
and
$\tilde{UB}(x_i^\top w^*) \le UB(x_i^\top w^*)$.
\end{proof}

% \subsection{Proof of Theorem 4}
% \label{subsec:proof_of_Theorem4}
% \begin{proof}
%   Obvious.
% \end{proof}

% \subsection{Proof of Theorem 5}
% \label{subsec:proof_of_Theorem5}
% \begin{proof}
%   Obvious.
% \end{proof}

\subsection{Proof of Theorem 6}
\label{subsec:proof_of_Theorem6}
We first construct
the
% compact
the dual optimal solution region
$\tilde{\Theta}_{\alpha^*}$.
\begin{lemm}
 \label{lemm:l1l1_dual_region}
 For an arbitrary pair of primal feasible solution
 $\hat{w} \in {\rm dom}P_\lambda$
 and dual feasible solution
 $\hat{\alpha} \in {\rm dom}D_\lambda$,
 the dual optimal solution region
 is written as
 \begin{align*}
  % \label{eq:l1l1_dual_region}
  % &\forall \hat{w} \in {\rm dom} P_{\lambda} ,
  % \forall \hat{\alpha} \in {\rm dom} D_{\lambda} , \nonumber \\
  &\Theta_{\alpha^*} :=
  \Set{ \forall i ~ y_i \alpha_i \in [0, 1]  |
  y^\top \hat{\alpha} \le
  y^\top \alpha \le
  P_{\lambda}(\hat{w}) }.
 \end{align*}
\end{lemm}

\begin{proof}[Proof of Lemma \ref{lemm:l1l1_dual_region} ]
 From the optimality and weak duality
 $y^\top \hat{\alpha} \le y^\top \alpha^*$
 and
  $ y^\top \alpha^*  \le P_{\lambda} (\hat{w}) $,
 respectively.
 Therefore,
 \begin{align*}
  &\alpha^* \in \hat{\Theta}_{\alpha^*} :=
  \Set{\alpha \in {\rm dom} D_{\lambda} |
   y^\top \hat{\alpha} \le y^\top \alpha \le P_{\lambda}(\hat{w}) }.
  %  \\
  % & ~~~~~~ \forall \hat{\alpha} \in {\rm dom} D_{\lambda},
  % \forall \hat{w} \in {\rm dom} P_{\lambda}.
 \end{align*}
 Noting that
 $\hat{\Theta}_{\alpha^*} \subseteq \Theta_{\alpha^*}$,
 $\alpha^*  \in \Theta_{\alpha^*}$.
\end{proof}

\begin{proof}[Proof of Theorem \ref{theo:l1l1_feature_screen}]
 From Lemma \ref{lemm:l1l1_dual_region},
 \begin{align*}
  X_{:j}^\top \alpha^* \ge LB(X_{:j}^\top \alpha^*) := \min_{\alpha \in \Theta_{\alpha^*}} X_{:j}^\top \alpha
 \end{align*}
  Moreover,
 \begin{align*}
  LB(X_{:j}^\top \alpha^*) = \min_{\alpha \in \Theta_{\alpha^*}} Z_{:j}^\top \alpha_{y},
 \end{align*}
  where $\alpha_{y} $
  $:= [y_1 \alpha_1, \ldots, y_n \alpha_n]^\top$.
 Let us define three $n$-dimensional vectors
 $\bar{\alpha}^{(1)}$,
 $\bar{\alpha}^{(2)}$
 and
 $\bar{\alpha}^{(3)}$
 as follows:
 \begin{align*}
  \bar{\alpha}^{(1)}_i &:= \mycase{
  y_i & (Z^\prime_{ij} < 0) \\
  0 & (\text{otherwise}),
  }
  \\
  \bar{\alpha}^{(2)}_i &:= \mycase{
  y_i & (Z^\prime_{ij} \le Z^\prime_{l_q j}) \\
  y_i (y^\top \hat{\alpha} - l_q)  & (Z^\prime_{ij} = Z^\prime_{(l_{q}+1) j}) \\
  0 & (otherwise),
  }
  \\
  \bar{\alpha}^{(3)}_i &:= \mycase{
  y_i & (Z^\prime_{ij} \le Z^\prime_{u_q j}) \\
  y_i (P_\lambda(\hat{w}) - u_q)  & (Z^\prime_{ij} = Z^\prime_{(u_{q}+1) j}) \\
  0 & (otherwise).
  }
 \end{align*}
 If
 $l_q + 1 \le n_{Z^\prime_{: j}} \le u_q$,
 then
 $\bar{\alpha}^{(1)}$
 is an element of
 $\Theta_{\alpha^*}$
 and minimizes $X_{:j}^\top \alpha$.
 If
 $n_{Z'_{:j}} < l_q+1$
 then
 $\bar{\alpha}^{(1)}  \not \in \Theta_{\alpha^*}$,
 $\bar{\alpha}^{(2)} $
 is an element of
 $\Theta_{\alpha^*}$
 and
 minimizes $X_{:j}^\top \alpha$
 because
 $y^\top \bar{\alpha}^{(2)} = y^\top \hat{\alpha}$.
 If
 $ m_{Z'_{:j}} > u_q$
 then
 $\bar{\alpha}^{(1)}  \not \in \Theta_{\alpha^*}$,
 meaning that
 $\bar{\alpha}^{(3)}$
 is an element of
 $\Theta_{\alpha^*}$
 and
 minimizes $X_{:j}^\top \alpha$
 because
 $ y^\top \bar{\alpha}^{(3)} = P_{\lambda}$.
 Therefore,
 \begin{align*}
    LB (X_{:j}^\top \alpha^* ) :=
    \begin{cases}
    \sum_{i=1}^{l_q}
    Z'_{ij} + (y^\top \hat{\alpha} - l_q) Z'_{(l_q+1) j}
    & (n_{Z'_{:j}} < l_q + 1), \\
    \sum_{i=1}^{u_q} Z'_{ij} +
    (P_{\lambda}(\hat{w}) - u_q),
     Z'_{u_q j}
    & (n_{Z'_{:j}} > u_q ) \\
    \sum_{i=1}^{n} \min \{0, Z'_{ij} \}
    & ({\rm otherwise}),
    \end{cases}
 \end{align*}
 Similarly,
 from Lemma \ref{lemm:l1l1_dual_region},
 \begin{align*}
  X_{:j}^\top \alpha^* \le UB(X_{:j}^\top \alpha^*) := \max_{\alpha \in \Theta_{\alpha^*}} X_{:j}^\top \alpha
 \end{align*}
 Moreover,
 \begin{align*}
  UB(X_{:j}^\top \alpha^*) = \max_{\alpha \in \Theta_{\alpha^*}} Z_{:j}^\top \alpha_{y}.
 \end{align*}
 Let us define three $n$-dimensional vectors
 $\bar{\alpha}^{(4)}$,
 $\bar{\alpha}^{(5)}$
 and
 $\bar{\alpha}^{(6)}$
 as follows:
 \begin{align*}
  \bar{\alpha}^{(4)}_i &:= \mycase{
  y_i & (Z^\prime_{ij} > 0) \\
  0 & (\text{otherwise}),
  }
  \\
  \bar{\alpha}^{(5)}_i &:= \mycase{
  y_i & (Z^\prime_{ij} \ge Z^\prime_{(n -l_q) j}) \\
  y_i (y^\top \hat{\alpha} - l_q)  & (Z^\prime_{ij} = Z^\prime_{(n-l_{q}-1) j}) \\
  0 & (otherwise),
  }
  \\
  \bar{\alpha}^{(6)}_i &:= \mycase{
  y_i & (Z^\prime_{ij} \ge Z^\prime_{(n-u_q) j}) \\
  y_i (P_\lambda(\hat{w}) - u_q)  & (Z^\prime_{ij} = Z^\prime_{(n -u_{q}-1) j}) \\
  0 & (otherwise).
  }
 \end{align*}
 If
 $l_q + 1 \le p_{Z'_{:j}} \le u_q $,
 then
 $\bar{\alpha}^{(4)}$
 is an element of
 $\Theta_{\alpha^*}$
 and maximizes $X_{:j}^\top \alpha$.
 If
 $ p_{Z'_{:j}} < l_q+1$
 then
 $\bar{\alpha}^{(4)}  \not \in \Theta_{\alpha^*}$,
 $\bar{\alpha}^{(5)} $
 is an element of
 $\Theta_{\alpha^*}$
 and
 maximizes $X_{:j}^\top \alpha$
 because
 $y^\top \bar{\alpha}^{(5)} = y^\top \hat{\alpha}$.
 If
 $ p_{Z'_{:j}} > u_q$
 then
 $\bar{\alpha}^{(4)}  \not \in \Theta_{\alpha^*}$,
 meaning that
 $\bar{\alpha}^{(6)}$
 is an element of
 $\Theta_{\alpha^*}$
 and
 maximizes $X_{:j}^\top \alpha$
 because
 $ y^\top \bar{\alpha}^{(6)} = P_{\lambda}$.
 Therefore,
   \begin{align*}
    UB (X_{:j}^\top \alpha^* ) :=
    \begin{cases}
    \sum_{i=n - l_q}^{n}
    Z'_{ij} + (y^\top \hat{\alpha} - l_q) Z'_{(n-l_q-1) j}
    & \!\!\!\!\!\!\!\! (p_{Z'_{:j}} < l_q + 1), \\
    \sum_{i=n-u_q}^{n} Z'_{ij} +
    (P_{\lambda}(\hat{w}) - u_q)
     Z'_{(n - u_q -1) j}
    & \!\!\!\! (p_{Z'_{:j}} > u_q ), \\
    \sum_{i=1}^{n} \max \{0, Z'_{ij} \}
    & \!\!\!\! ({\rm otherwise}).
    \end{cases}
  \end{align*}

  On the other hand,
  from KKT condition(\ref{eq:kkt1}),
  \begin{align}
    \label{eq:kkt1_l1reg}
    \frac{1}{\lambda n} X_{:j}^\top \alpha^* \in \begin{cases}
      \frac{w^*_j}{|w^*_j|} & (w^*_j \neq 0) \\
      [-1, 1] & ({\rm otherwise}).
    \end{cases}
  \end{align}
  Therefore,
  if
  $LB (X_{:j}^\top \alpha^* ) < -\lambda n $ and
  $UB (X_{:j}^\top \alpha^* ) > \lambda n $ then
  $w^*_j = 0$.
\end{proof}

\subsection{Proof of Theorem 7}
\label{subsec:proof_of_Theorem7}
First, we construct
the primal optimal solution region
${\Theta_{w^*}}$.
\begin{lemm}
 \label{lemm:l1l1_primal_region}
 The primal optimal solution region
 $\Theta_{w^*}$
 is given
 $\forall \hat{w} \in {\rm dom} P_{\lambda}$
 as
 \begin{align}
  \label{eq:l1l1_primal_region}
  \Theta_{w^*} =
  \Set{ {w} \in {\rm dom} P_{\lambda} |
  \lambda \|w\|_1 + g_{\ell}(\hat{w})^\top w \le k },
 \end{align}
 where
 $g_{\ell}(w) := \frac{1}{n} \sum_{i \in [n]} g_{\ell_i}({w}) $.
\end{lemm}

\begin{proof}
From Proposition B.24 in \cite{Bertsekas99a},
\begin{align*}
 (\lambda g_{\psi}(w^*) + g_{\ell}(w^*))^\top ( w^* - \hat{w} ) \le 0 , \forall \hat{w} \in {\rm dom} P_{\lambda},
\end{align*}
where $ g_{\psi}(w) \in \partial \psi (w)$.
Form the convexity of $\ell_i $ for $i \in [n]$ and
the definition of subgradient
 \begin{align*}
  &
  \ell_i(w^*)
  \ge
  \ell_i(\hat{w})
  +
   g_{\ell_i}(\hat{w})
  (w^* - \hat{w}),  \forall \hat{w} \in {\rm dom} P_{\lambda}
  \\
  &
  \ell_i(\hat{w})
  \ge
  \ell_i(w^*)
  +
  g_{\ell_i}(w^*)
  (\hat{w} - w^*),  \forall \hat{w} \in {\rm dom} P_{\lambda},
 \end{align*}
 and thus,
 $
 g_{\ell_i}(w^*)^\top (w^* - \hat{w})
 \ge
 g_{\ell_i}(\hat{w})^\top (w^* - \hat{w}) ,  \forall \hat{w} \in {\rm dom} P_{\lambda}
 $.
 Therefore,
 $\forall \hat{w} \in {\rm dom} P_{\lambda}$,
 \begin{align*}
  \lambda g_{\psi}(w^*)^\top w^*  +
  g_{\ell}(\hat{w})^\top w^*
  \le \lambda g_{\psi}(w^*)^\top \hat{w} + g_{\ell}(\hat{w})^\top {\hat{w}}.
 \end{align*}
 Since
 $
 g_{\psi}(\hat{w} )^\top \hat{w} = \| \hat{w} \|_1
  = \max_{s \in [-1,1]^d} s^\top \hat{w}
 $
 and
 $
  g_{\psi}(w^*) \in [-1,1]^d,
 $
 we have
 \begin{align*}
  \lambda g_{\psi}(w^*)^\top \hat{w} \le
  \lambda g_{\psi}(\hat{w} )^\top \hat{w}.
 \end{align*}
 By combining these results,
 \begin{align*}
  \lambda \| w^* \|_1  + g_{\ell}(\hat{w} )^\top w^*  \le k ,
  ~~ \forall \hat{w} \in {\rm dom} P_{\lambda}.
 \end{align*}
\end{proof}

\begin{proof}[Proof of Theorem \ref{theo:l1l1_sample_screen}]
From Lemma \ref{lemm:l1l1_primal_region},
 \begin{align*}
  x_i^\top w^* \ge LB (x_i^\top w^*)
  := \min_{ w \in \Theta_{w^*}} x_i^\top w.
 \end{align*}
 Using a Lagrange multiplier
 $\mu > 0$,
 \begin{align}
  \label{eq:appA:LB}
  LB (x_i^\top w^*)
  &= \min x_i^\top w  ~~{\rm s.t}~~ w \in \Theta_{w^*} \\
  \nonumber
  &= \min_{w} \max_{\mu > 0}
  \{x_i^\top w + \mu( \lambda \|w\|_1 + g_{\ell}(\hat{w})^\top w - k ) \} \\
  \nonumber
  &= \max_{\mu > 0}  \{ \mu k +
  \min ( \underbrace{x_i^\top w + \mu \lambda \|w\|_1
  + \frac{\mu}{n} g_{\ell}(\hat{w})^\top w }_{L(w)} ) \}
 \end{align}
 Since
 $0 \in \partial L$,
 which is written as
 $\partial L = x_i + \mu \lambda \partial \psi(w) + \frac{\mu}{n} g_{\ell}(\hat{w})$,
 we have
 \begin{align}
  \label{eq:app:lambda_g}
  \mu \lambda g_{\psi}(w) = -  x_i - \frac{\mu}{n} g_{\ell}(\hat{w})
 \end{align}
 Substituting
 $\mu \lambda \|w\| = -  x_i^\top w - \frac{\mu}{n} g_{\ell}(\hat{w})^\top w$
 into
 \eq{eq:appA:LB},
 \begin{align*}
  &LB (x_i^\top w^*) =
  \max_{\mu > 0}  \{ \mu k \}
  ~~{\rm s.t.}~~
  \| - \frac{1}{\lambda}  x_i^\top w -
  \frac{\mu}{\lambda n} g_{\ell}(\hat{w})^\top w \|_{\infty}
  \le \mu,
 \end{align*}
 where the constraint comes from \eq{eq:app:lambda_g}.
 %$\| - \frac{1}{\lambda}  x_i^\top w - \frac{\mu}{\lambda n} g_{\ell}(\hat{w})^\top w \|_{\infty}\le \mu$.
 %
 Similarly,
 since
 \begin{align*}
  x_i^\top w^* \le UB (x_i^\top w^*)
  := \max_{ w \in \Theta_{w^*}} x_i^\top w
  = - \min_{ w \in \Theta_{w^*}} x_i^\top w,
 \end{align*}
\begin{align*}
  &UB (x_i^\top w^*) =
  \max_{\mu > 0}  \{ \mu k \}
  ~~{\rm s.t.}~~
  \| \frac{1}{\lambda}  x_i^\top w -
  \frac{\mu}{\lambda n} g_{\ell}(\hat{w})^\top w \|_{\infty}
  \le \mu.
\end{align*}
\end{proof}

\section{Safe keeping by using KKT optimality conditions}
In this appendix,
we describe
another type of safe keeping approaches
based on KKT optimality conditions.
\begin{theo}
  \label{theo:safe_feature_keeping_kkt}
 For an arbitrary pair of primal feasible solution
 $\hat{w} \in {\rm dom}P_\lambda$
 and dual feasible solution
 $\hat{\alpha} \in {\rm dom}D_\lambda$,
 \begin{align*}
  LB(X_{:j}^\top \alpha^*) < - \lambda n
  ~{\rm and}~
  \lambda n < UB(X_{:j}^\top \alpha^*)
  ~\Rightarrow~
  w_j^* \neq 0
 \end{align*}
 for $j \in [d]$,
 where
 \begin{align*}
   LB(X_{:j}^\top \alpha^*) &:=
    X_{:j}^\top \hat{\alpha} -
   \|X_{:j}\|_2 \sqrt{2 n G_\lambda(\hat{w}, \hat{\alpha})/\gamma}, \\
   UB(X_{:j}^\top \alpha^*) &:=
    X_{:j}^\top \hat{\alpha} +
    \|X_{:j}\|_2 \sqrt{2 n G_\lambda(\hat{w}, \hat{\alpha})/\gamma}.
 \end{align*}
\end{theo}

\begin{proof}
In the case that $D_{\lambda}$ is $\gamma/n$-strongly concave,
$X_{:j}^\top \alpha^*$
is bounded from below and above respectively by the following lower and upper bounds:
 \begin{align*}
   LB(X_{:j}^\top \alpha^*) &:=
    X_{:j}^\top \hat{\alpha} -
   \|X_{:j}\|_2 \sqrt{2 n G_\lambda(\hat{w}, \hat{\alpha})/\gamma}, \\
   UB(X_{:j}^\top \alpha^*) &:=
    X_{:j}^\top \hat{\alpha} +
    \|X_{:j}\|_2 \sqrt{2 n G_\lambda(\hat{w}, \hat{\alpha})/\gamma}.
 \end{align*}

On the other hand,
in the case of our specific regularization term
\eq{eq:regularization-term},
from
KKT optimality condition \eq{eq:kkt1},
if  $ -\lambda n < X_{:j}^\top \alpha^* < \lambda n $ then
$w^*_j \neq 0$.

Therefore,
 \begin{align*}
  LB(X_{:j}^\top \alpha^*) < - \lambda n
  ~{\rm and}~
  \lambda n < UB(X_{:j}^\top \alpha^*)
  ~\Rightarrow~
  w_j^* \neq 0
 \end{align*}

\end{proof}

Similarly,
we can develop safe sample keeping
based on KKT optimiality condition.
\begin{theo}
  \label{theo:safe_sample_keeping_kkt}
 For an arbitrary pair of primal feasible solution
 $\hat{w} \in {\rm dom}P_\lambda$
 and dual feasible solution
 $\hat{\alpha} \in {\rm dom}D_\lambda$,
 if $\ell_i $ is smoothed hinge loss then,
 for $y_i = +1$,
 \begin{align*}
   1- \gamma < LB(x_i^\top w^*)  ~{\rm and} ~ UB(x_i^\top w^*) < 1
   ~\Rightarrow~
  \alpha^*_i \not \in \{0, +1\},
 \end{align*}
 and, for $y_i = -1$,
 \begin{align*}
  -1 < LB(x_i^\top w^*) ~{\rm and} ~  UB(x_i^\top w^*) < \gamma -1
  ~\Rightarrow~
  \alpha^*_i \not \in \{-1, 0\}.
 \end{align*}
 If $\ell_i $ is smoothed $\veps$-insensitive loss then
  \begin{align*}
  -\gamma + y_i - \veps < LB(x_i^\top w^*) ~&{\rm and} ~  UB(x_i^\top w^*) < y_i - \veps \\
  &~{\rm or} ~ \\
  y_i + \veps < LB(x_i^\top w^*) ~&{\rm and} ~ UB(x_i^\top w^*) < \gamma + y_i + \veps \\
  ~\Rightarrow~  \alpha^*_i &\not \in \{-1, 0, +1\},
  \end{align*}
  for $j \in [d]$,
  where
 \begin{align*}
  LB(x_i^\top w^*) &= x_i^\top \hat{w} -
  \|x_i\|_2 \sqrt{2 G_{\lambda}(\hat{w}, \hat{\alpha}) / \lambda},  \\
  UB(x_i^\top w^*) &= x_i^\top \hat{w} +
  \|x_i\|_2 \sqrt{2 G_{\lambda}(\hat{w}, \hat{\alpha}) / \lambda }.
 \end{align*}
 \end{theo}

\begin{proof}
In the case that $P_{\lambda}$ is $\lambda$-strongly convex,
$x_i^\top w^*$
is bounded from below and above respectively by the following lower and upper bounds:
 \begin{align*}
  LB(x_i^\top w^*) &= x_i^\top \hat{w} -
  \|x_i\|_2 \sqrt{2 G_{\lambda}(\hat{w}, \hat{\alpha}) / \lambda},  \\
  UB(x_i^\top w^*) &= x_i^\top \hat{w} +
  \|x_i\|_2 \sqrt{2 G_{\lambda}(\hat{w}, \hat{\alpha}) / \lambda }.
 \end{align*}

On the other hand,
from
KKT optimality condition \eq{eq:kkt1},
in the case of smoothed hinge loss
\eq{eq:smoothed_hinge},
if $y_i = +1$ and $1- \gamma < x_i^\top w^* < 1$ then $\alpha^*_i \in \{0, +1\}$,
if $y_i = 1$ and $ -1  < x_i^\top w^* <\gamma - 1$ then $\alpha^*_i \in \{-1, 0\}$.
Therefore,
 \begin{align*}
   &y_i = +1 ~{\rm and} ~
   1- \gamma < LB(x_i^\top w^*)  ~{\rm and} ~ UB(x_i^\top w^*) < 1 \\
   &~\Rightarrow~
  \alpha^*_i \not \in \{0, +1\}, \\
  &y_i = -1 ~{\rm and} ~
  -1 < LB(x_i^\top w^*) ~{\rm and} ~  UB(x_i^\top w^*) < \gamma -1 \\
  &~\Rightarrow~
  \alpha^*_i \not \in \{-1, 0\}.
 \end{align*}
Also, in the case of smoothed $\veps$-insensitive
\eq{eq:soft_insensitive},
if
$-\gamma + y_i - \veps < x_i^\top w^* < y_i - \veps $
or
$y_i + \veps < x_i^\top w^* < \gamma + y_i + \veps $
then,
$\alpha^*_i \not \in \{-1, 0, +1\}$
.
Therefore,
  \begin{align*}
  -\gamma + y_i - \veps < LB(x_i^\top w^*) ~&{\rm and} ~  UB(x_i^\top w^*) < y_i - \veps \\
  &~{\rm or} ~ \\
  y_i + \veps < LB(x_i^\top w^*) ~&{\rm and} ~ UB(x_i^\top w^*) < \gamma + y_i + \veps
  ~\Rightarrow~  \alpha^*_i \not \in \{-1, 0, +1\},
  \end{align*}

 \end{proof}

\clearpage

\bibliographystyle{unsrt}
\bibliography{paper}

\clearpage

\end{document}